\documentclass[10pt]{article}
\usepackage[accepted]{tmlr}

\usepackage{hyperref}
\usepackage{url}
\usepackage{amsmath,amsthm,amsfonts}
\usepackage{bbm}
\usepackage{natbib}
\usepackage{algorithm,algorithmic}
\usepackage{xcolor}
\usepackage{wrapfig}

\usepackage{tikz}
\usepackage{pgfplots}
\usepackage{placeins}
\pgfplotsset{compat=1.18}

\newtheorem{theorem}{Theorem}

\newtheorem{lemma}[theorem]{Lemma}
\newtheorem{proposition}[theorem]{Proposition}
\newtheorem{assumption}{Assumption}
\theoremstyle{definition}
\newtheorem{definition}{Definition}
\newtheorem{remark}{Remark}

\title{Effect of Random Learning Rate: Theoretical Analysis of SGD Dynamics in Non-Convex Optimization via Stationary Distribution}

\author{\name Naoki Yoshida \email n-yoshida@g.ecc.u-tokyo.ac.jp \\
      \addr The University of Tokyo, Bunkyo, Tokyo, Japan.
      \AND
      \name Shogo Nakakita \email nakakita@g.ecc.u-tokyo.ac.jp \\
      \addr The University of Tokyo, Bunkyo, Tokyo, Japan.
      \AND
      \name Masaaki Imaizumi \email imaizumi@g.ecc.u-tokyo.ac.jp \\
      \addr The University of Tokyo, Bunkyo, Tokyo, Japan. \\
      RIKEN Center for Advanced Intelligence Project, Chuo, Tokyo, Japan.
}

\newcommand{\R}{\mathbb{R}}
\newcommand{\N}{\mathbb{N}}

\newcommand{\mA}{\Theta}

\newcommand{\mW}{\mathcal{W}}

\newcommand{\Ep}{\mathbb{E}}

\newcommand{\mZ}{\mathcal{Z}}

\renewcommand{\hat}{\widehat}

\newcommand{\argmin}{\operatornamewithlimits{argmin}}

\usepackage{color}

\usepackage{subcaption}

\usepackage{mathtools}
\mathtoolsset{showonlyrefs=true}
\allowdisplaybreaks

\AtBeginEnvironment{algorithmic}{\let\AND\algoAND}

\def\month{08}  
\def\year{2025} 
\def\openreview{\url{https://openreview.net/forum?id=RPtKkNx9ZK}} 

\begin{document}

\maketitle
\begin{abstract}
We consider a variant of the stochastic gradient descent (SGD) with a random learning rate and reveal its convergence properties. SGD is a widely used stochastic optimization algorithm in machine learning, especially deep learning. Numerous studies reveal the convergence properties of SGD and its theoretically favorable variants. Among these, the analysis of convergence using a stationary distribution of updated parameters provides generalizable results. However, to obtain a stationary distribution, the update direction of the parameters must not degenerate, which limits the applicable variants of SGD. In this study, we consider a novel SGD variant, Poisson SGD, which has degenerated parameter update directions and instead utilizes a random learning rate. Consequently, we demonstrate that a distribution of a parameter updated by Poisson SGD converges to a stationary distribution under weak assumptions on a loss function. Based on this, we further show that Poisson SGD finds global minima in non-convex optimization problems and also evaluate the generalization error using this method. As a proof technique, we approximate the distribution by Poisson SGD with that of the bouncy particle sampler (BPS) and derive its stationary distribution, using the theoretical advance of the piece-wise deterministic Markov process (PDMP).
\end{abstract}

\section{Introduction}
Stochastic gradient descent (SGD) stands out as a widely employed optimization algorithm in machine learning. It falls under the category of stochastic optimization, where parameters are updated with randomness from the mini-batch sampling. SGD is valued for two main reasons in optimization: (i) it is memory-efficient and requires only low computational resources by updating parameters from a fraction of the training data at each iteration \citep{bottou1991stochastic}, and  (ii) models optimized with SGD have less generalization error than those optimized by other algorithms such as gradient descent (GD) for neural networks \citep{noisy-GD, anisotropic}.
Owing to these advantages, SGD has been one of the standard methods for training deep learning models \citep{hoffer2017train,keskar2016large,anisotropic}.

To understand the properties of SGD, the characteristics of parameters updated by SGD or its variants have been actively studied.
As for the usual SGD, \cite{handbook} surveyed the results about the convergence rate of SGD in convex and non-convex settings. It also mentions the global convergence property of SGD under the strong convexity setting.
\cite{SGD-diffusion,jastrzkebski2017three} clarified that the parameter updating process of SGD can be approximated by a stochastic differential equation. \cite{anisotropic,nguyen2019first} discussed the relation between the random noise of SGD and the escape efficiency from the sharp minima of the loss function.
One example of a variant of SGD is stochastic gradient Langevin dynamics (SGLD), which is an extension of SGD that adds Gaussian noise to the update formula of SGD.
\cite{Raginsky-2017} analyzed the dynamics of stochastic gradient Langevin dynamics (SGLD) as a variant of SGD and proved the parameters optimized by SGLD converge to the global minima of the generalization error.
As another example, \cite{jastrzkebski2017three,SGD-noise-gaussian1,SGD-noise-gaussian2} analyzed the dynamics of SGD with a constant learning rate under the assumptions that the noise of SGD on the gradient induced by the mini-batch sampling is isotropic, and derived the probability distribution of the parameters obtained by SGD.
\cite{Latz} analyze SGD both in the case of the constant learning rate and of the decreasing learning rate.

\begin{figure}
    \centering
    \begin{tikzpicture}
\begin{axis}[
    width=6cm, height=3cm,
    axis lines=middle,
    domain=-8:8, samples=200,
    xlabel={$\theta$}, ylabel={$L_{\textbf{z}}(\theta)$},
    xtick=\empty, ytick=\empty,
    title={\textbf{Non-convex Loss}},
]
    \addplot[blue, thick] {x^4 - 4*x^3 - 36*x^2 + 3000};
\end{axis}
\end{tikzpicture}
\hfill
\begin{tikzpicture}
\begin{axis}[
    width=6cm, height=3cm,
    axis lines=middle,
    domain=-8:8, samples=200,
    xlabel={$\theta$}, ylabel={Density},
    xtick=\empty, ytick=\empty,
    title={\textbf{Stationary distribution $\mu_\beta$} \\ (moderate temperature: small $\beta$)},
    title style={align=center}
]
    \addplot [
        domain=-8:8,
        samples=200,
        fill=purple,
        fill opacity=0.3,
        draw=purple,
        thick
    ]
    {exp(-(x^4 - 4*x^3 - 36*x^2)/1000)} \closedcycle;
\end{axis}
\end{tikzpicture}
\hfill
\begin{tikzpicture}
\begin{axis}[
    width=6cm, height=3cm,
    axis lines=middle,
    domain=-8:8, samples=200,
    xlabel={$\theta$}, ylabel={Density},
    xtick=\empty, ytick=\empty,
    title={\textbf{Stationary distribution $\mu_\beta$} \\ (low temperature: large $\beta$)},
    title style={align=center}
]
    \addplot [
        domain=-8:8, 
        samples=200, 
        fill=purple, 
        fill opacity=0.3,  
        draw=purple, 
        thick
    ] 
    {exp(-(x^4 - 4*x^3 - 36*x^2)/100)} \closedcycle;
\end{axis}
\end{tikzpicture}

    \caption{Approach using stationary distributions. Under the non-convex loss $L_{\textbf{z}}(\theta)$ of a parameter $\theta$ (the left panel), we consider a distribution of parameters $\theta$ generated from an algorithm with an inverse temperature $\beta$, which converges to the stationary distribution $\mu_\beta$ with sufficiently large number of iterations (the middle panel). By varying $\beta$, the distribution $\mu_\beta$ is concentrated on the global optimum of $L_{\textbf{z}}(\theta)$ (the right panel).}
    \label{fig:overview_stationary}
\end{figure}

Among the methods analyzing the properties of SGD, one of the most general approaches is to study a \textit{stationary distribution} of parameters updated by SGD and its variants.
The stationary distribution is a distribution that remains unchanged when the parameter is updated by one step.
Analysis based on stationary distributions is an approach that aims to characterize the probability distribution rather than the exact values of the updated parameters. This method has been employed by researches such as \cite{Raginsky-2017} or \cite{strongly-convex1} for analyzing machine learning models with non-convex loss functions.
Specifically, it involves demonstrating that a distribution of parameters updated by the algorithm under sufficient iterations converges to a \textit{stationary distribution} that assigns mass across the entire parameter space. Subsequently, by adjusting an inverse temperature parameter that affects its shape, the stationary distribution can be concentrated around the global minimum. See Figure \ref{fig:overview_stationary} for illustration.
It is useful in theoretical analysis, because (i) it can analyze the global convergence of the optimization algorithm, and (ii) it can be applied to a wide range of loss functions regardless of its convexity.

Despite the above advantages, there are not many SGD variants to which stationary distribution analysis can be applied. 
This is because, to use the analysis by a stationary distribution, it is required that the direction of parameter updates by an algorithm does not degenerate; in other words, there must be no directions that are not being explored.
Examples of such variants are (i) SGLD \citep{SGLD,Langevin1,Langevin2}, which adds a Gaussian noise to the parameter update of SGD and (ii) Gaussian SGD \citep{jastrzkebski2017three,SGD-noise-gaussian1,SGD-noise-gaussian2}, which assumes that the noise of SGD on the gradient induced by the mini-batch sampling is non-degenerate Gaussian. 
In contrast, the parameter update of SGD degenerates in many practical cases, such as deep learning \citep{anisotropic,nguyen2019first,simsekli2019tail}. We remark that we focus on the degeneracy of the update direction of SGD, not on the distribution of it since there is no clear agreement that gradient noise follows a particular distribution (the definition of degeneracy is in Remark \ref{degeneracy}).
Hence, there is a gap between the variants of SGD considered in the theoretical analysis and the empirical facts about SGD.
This gap fosters the following question:
\begin{align*}
    &\mbox{\textit{Do parameters optimized by a variant of SGD have a stationary distribution}}\\
    &\mbox{\textit{even if the update direction degenerates - and if so, what is the form of it?}}
\end{align*}

\subsection{Our Contribution}
We theoretically prove that a variant of SGD has a stationary distribution even if the update direction degenerates.
Specifically, we develop a novel SGD variant with a \textit{random learning rate}, which follows the Poisson process depending on a mini-batch gradient. We call the variant \textit{Poisson SGD}, and prove that the distribution of a parameter updated by Poisson SGD converges to a stationary distribution.
As a result, we provide a positive answer to the question posed above: even with a degenerated parameter update, we can construct a variant of SGD that reaches a stationary distribution by using a random learning rate. 

Our specific contributions are as follows.
We consider the empirical risk minimization problem and prove the following results under weak assumptions on the loss function such as absolute continuity: (i) the distribution of the parameters updated by Poisson SGD converges to a stationary distribution, and (ii) an output of Poisson SGD converges to the global minima of the empirical risk, applying the stationary distribution while controlling the inverse-temperature parameter.
Furthermore, we evaluate the generalization error of the updated parameter for prediction with unseen data by studying an expectation of the risk function in terms of the obtained stationary distribution.

On the technical side, 
we utilize an algorithm called the Bouncy Particle Sampler (BPS) to demonstrate the convergence to the stationary distribution by Poisson SGD. 
BPS is a piecewise deterministic Markov process (PDMP) that achieves ergodicity using stochastically occurring jumps \citep{PDMP1,PDMP2}. 
In our proof, we show that the distribution of parameters updated by Poisson SGD can be well approximated by that of BPS, and we concretely construct the stationary distribution using the theory of BPS.

\subsection{Related Work}
There are many works which investigate the stationary distribution of SGD or its variants.
\cite{strongly-convex1,strongly-convex2} derived the stationary distribution of the parameters obtained by SGD when the loss function is strongly-convex, through the theories about Markov processes.
The parameters obtained through the SGLD algorithm are theoretically proven to converge to the Gibbs distribution and generalize well \citep{Raginsky-2017}.
\cite{SGD-noise-gaussian1} and \cite{SGD-noise-gaussian2} assumed the noise of SGD is Gaussian whose covariance matrix is constant and approximate the process of optimization through SGD by Ornstein-Uhlenbeck process and derive its stationary distribution. Gradient Langevin dynamics (GLD), which is a full-batch version of SGLD, can also be seen as a variant of SGD which assumes that the noise of SGD is Gaussian with a covariance matrix of constant multiples of the identity matrix.
Like SGLD, it converges to a stationary distribution even in non-convex scenarios \citep{Langevin1,Langevin2}. 

In terms of a random learning rate, there are several empirical studies. \cite{Musso} investigated the dynamics of SGD with a random learning rate by analyzing the stochastic differential equation and its Fokker-Planck equation. \cite{Blier} showed experimentally that SGD with random learning rates performs well in optimizing deep neural networks. Note that these studies and ours have several major differences. The first difference is in the design of a learning rate. Our method considers Poisson processes, whereas existing methods consider uniform distributions and heterogeneous learning rates for each subneural network. The second difference is the objective of the study. We aim to evaluate global convergence, while existing studies aim at interpretability, speed of convergence, etc., and have very different motivations.

As for BPS, \citet{Deligiannidis2017} and \citet{Durmus-et-al-2020} proved that the parameters updated by continuous-time BPS converge to a stationary distribution and derived the concrete form of the stationary distribution and its convergence rate. \citet{discreteBPS} clarified the relation between discrete-time BPS and continuous-time BPS.

\subsection{Notation}
For a natural number $a \in \N$, we define $[a] := \{1,2,...,a\}$.
For a real $z \in \R$, $\lfloor z\rfloor$ denotes the largest integer which is no more than $z$.
$I_d$ is a $d$-dimensional identity matrix. 
$\langle a,b\rangle$ means the inner product in Euclidean space, \textit{i.e.}, sum of the product of each component. 
$\|\cdot\|_1$ and $\|\cdot\|$ mean the vector norms which represent 1-norm and 2-norm respectively.
$\mathbb{S}^{d-1}$ is a unit sphere in $\R^d$.
For probability measures $P,P'$ on $\R^d$ and $p\in [1,\infty]$, the $p$-Wasserstein distance is defined as 
    $\mW_p (P,P') :=  \inf_{\pi \in \Pi(P,P')} (\int_{\R^d} \|z - z'\|_p^p d \pi(z,z'))^{1/p}$,
where $\Pi(P,P')$ is a set of coupling measures between $P$ and $P'$.
$\|P - P'\|_{\mathrm{TV}}$ denotes the total variation of $P - P'$.
For a compact set $\Theta$, we denote $\mathrm{diam}(\Theta)=\sup_{\theta_1,\theta_2\in\Theta} \|\theta_1-\theta_2\|$.
For a random variable $X\in\mathcal{X}$, $\mathbb{E}_{X}[X]$ denotes the expected value with regard to $X$, \textit{i.e.}, $\int_{\mathcal{X}} xd\mu_X(dx)$, where $\mu_X$ is the probability measure of $X$.
$\mathbbm{1}[\cdot]$ denotes an indicator function, which takes $1$ if the condition in the bracket is satisfied, and $0$ otherwise.
We denote $a_+=\max\{0,a\}$.
For $a\in\mathbb{R}$ and $\theta\in\mathbb{R}^d$, $\theta   \,\mathrm{mod}\,a$ means calculating modulo $a$ for every element of $\theta$.
$\mathrm{B}:\mathbb{R}\times\mathbb{R}\rightarrow\mathbb{R}$ denotes the beta function, \textit{i.e.}, $\mathrm{B}(x,y)=\int_{0}^1 t^{x-1}(1-t)^{y-1}dt$. 
$\Gamma:\mathbb{R}\rightarrow\mathbb{R}$ denotes the gamma function, \textit{i.e.}, $\Gamma(z)=\int_{0}^{\infty} t^{z-1}e^{-t}dt$.
 
\section{Preliminary}

\subsection{Problem Setup: Empirical Risk Minimization} \label{sec:setup}

We consider the following stochastic optimization problem.
Let $\mZ$ be a compact sample space, and consider a probability measure $P_*$ on $\mZ$.
Suppose that we observe $n$ samples $\textbf{z} = \{z_1,...,z_{n}\} \subset \mZ$, that are independently and identically generated from the measure $P_*$.
Using the samples, we consider an empirical risk with a loss function.
Let $\Theta \subset \R^d$ be a parameter space.
With a non-convex loss function $\ell: \mathcal{Z}\times\Theta\rightarrow \mathbb{R}$, we consider the following empirical risk with the samples:
\begin{align}
    &L_{\textbf{z}}(\theta) = \frac{1}{n}\sum_{i=1}^{n} \ell(z_i; \theta), ~ \theta \in \Theta. \label{def:empirical_risk}
\end{align}
In this study, we consider a torus $(\mathbb{R}/a\mathbb{Z})^d$ as the parameter space $\Theta$. Here, $a > 0$ is a periodicity parameter.
Also, we define $W := \mathrm{diam}(\Theta)=a\sqrt{d}$.

Our goal is to find a global minimum of the empirical risk $L_{\textbf{z}}(\cdot)$, which is defined as a parameter $\theta^* \in \Theta$ that satisfies
    $L_{\textbf{z}}(\theta^*) \in \min_{\theta' \in \Theta} L_{\textbf{z}}(\theta')$.

\begin{remark}[Motivation of torus parameter space]

We discuss the motivation for employing the torus.

First, the use of a torus parameter space is a widely adopted technique for simplifying otherwise unnecessarily intricate arguments in the analysis of stationary distributions.
Specifically, if we set $\Theta$ as Euclidean spaces or hypercubes, the analysis often becomes technically cumbersome due to boundary effects or additional projection steps after updating the parameters. 
In contrast, the torus provides a setting in which these peripheral issues can be bypassed, enabling a focus on the mathematically essential aspects of the argument. As a result, the torus has been utilized in a wide range of theoretical studies. 
In the context of stochastic processes, \cite{torus2, torus3, torus4} considers the torus parameter spaces, and \cite{torus1} does in the same way in the analysis of optimization. 
The generality of these assumptions has been broadly acknowledged in the literature.

Second, when the periodicity parameter $a$ of the torus is sufficiently large, the torus becomes practically indistinguishable from the Euclidean space in terms of behavior relevant to applications. 
As we will show later, our results hold for all values of $a$. 
Therefore, by choosing $a$ to be sufficiently large, the analysis becomes nearly equivalent to that conducted in $\R^d$ \citep{torus-textbook}.
\end{remark}

\subsection{Gradient Descent Algorithm and Variants}

To find the global minimum $\theta^*$ as defined in Section \ref{sec:setup}, we often use the optimization algorithm called stochastic gradient descent (SGD) with momentum.

\subsubsection{General Form of Stochastic Gradient Descent}

We give a formal definition of SGD with a momentum term associated with empirical risk $L_{\textbf{z}}(\theta)$ in \eqref{def:empirical_risk}.
Let $K \in \N$ be the number of iterations.
The SGD with momentum generates a sequence of $\Theta$-valued random parameters $\theta_1,...,\theta_K$ and $\R^d$-valued random vectors $v_1,...,v_K$, by the following procedure.

Let $\theta_0 \in \Theta$ be an arbitrary parameter for the initialization, $v_0 \in \R^d$ be an initial velocity vector, and $m \in [n]$ be a number of sub-samples, i.e., the batch size.
Suppose that we observe $n$ samples $\textbf{z} := \{z_1,...,z_n\}$, i.e., the full-batch.
For $k=1,...,K$, we uniformly sample $m$ integers $I^{(k)}=\{i_1, i_2, ..., i_m\}$ from $[n]$ with replacement, which is called mini-batch sampling with the batch-size $m$.
We define an associated mini-batch risk as
\begin{align}
     \hat{L}_{\textbf{z}}^{(k)} (\theta) := \frac{1}{m}\sum_{i\in I^{(k)}}  \ell(z_i;\theta). \label{def:batch_loss}
\end{align}
Then, with initial values $\theta_0 \in \Theta$ and $v_0 \in \R^d$, the SGD with momentum generates the parameter and the velocity vector by the following recursive formula for $k=1,...,K$:
\begin{align}
\theta_{k}&=\theta_{k-1}+\eta_k v_{k-1}, \mbox{~~~and~~~}
v_k =v_{k-1}-\alpha_k \nabla \hat{L}_{\textbf{z}}^{(k)}(\theta_k) \label{velocity_udpate},
\end{align}
where $\eta_k > 0$ is a learning rate and $\alpha_k\in\mathbb{R}$ is a momentum coefficient. 
This form is generic and can be identical to other forms of SGD with momentum \citep{momentum-1,momentum-2} by adjusting the parameters $\eta$ and $\alpha$.

\begin{remark}[Gradient Noise] \label{remark:grad_noise}
For the sake of technical discussions below, we define a notion of \textit{gradient noise} $\xi_k^{(m,n)}(\theta) := \nabla \hat{L}_{\textbf{z}}^{(k)}(\theta) - \nabla L_{\textbf{z}}(\theta)$ for $k=1,...,N$ and $\theta \in \Theta$, which is caused by sub-sampling of the SGD.
If one assumes that $\xi_k^{(m,n)}(\theta)$ follows a centered Gaussian distribution with an identity covariance, the SGD corresponds to the gradient Langevin dynamics (GLD). However, it is empirically observed that the covariance matrix of the gradient noise often degenerates \citep{anisotropic,Langevin-SGD1, Langevin-SGD2}. In addition, there is still much discussion on a distribution that gradient noise follows, e.g. \cite{simsekli2019tail} and \cite{revisiting} reports the non-Gaussianity of the gradient noise in empirical studies. Due to these situations, we do not consider a full-rank covariance matrix nor a particular distribution of the gradient noise.

\end{remark}

\begin{remark}[Degeneracy of the gradient noise]\label{degeneracy}
We briefly explain the \textit{degeneracy} of the gradient noise. Since the covariance matrix of the gradient noise with the batch-size $m$ is written as
    $\frac{1}{m}\sum_{i=1}^m (\nabla \ell(z_i;\theta)-\nabla L_{\textbf{z}}(\theta))(\nabla \ell(z_i;\theta)-\nabla L_{\textbf{z}}(\theta))^\top$,
where each term in the sum is a rank-1 matrix, the rank of a total covariance matrix is no greater than $m$. Hence, in the over-parameterized models like neural networks, the matrix becomes rank-deficient, which we refer to as degeneracy of the noise.

\end{remark}

\section{Our SGD Variant: Poisson SGD}

In this section, we introduce our algorithm, \textit{Poisson SGD}, which is a variant of SGD with a random learning rate $\eta$ and momentum coefficient $\alpha$.
We design our method so that the parameter can search the whole parameter space owing to the design.

We describe the random learning rate.
In preparation, we define the following exponential distribution function with a function $f: \Theta \to \R^d$ and parameters $\theta\in\Theta,v \in \mathbb{S}^{d-1}$:
\begin{align}
    E(f(\cdot),\theta,v) := \exp\left(-\int_{0}^{t}\{ \max\{\langle  f(\theta + r v),v \rangle, 0\}+C_P\}dr\right),
\end{align}
where $C_P>0$ is some constant.
Then, for each update $k = 1,...,K$, we design the random learning rate $\eta_k$ following the exponential distribution:
\begin{align}
    P(\eta_k\geq t) = E(\beta\nabla \hat{L}_{\textbf{z}}^{(k)}(\cdot), \theta_{k-1},v_{k-1})\label{update_eta},
\end{align}
where $\beta>0$ is the hyperparameter of Poisson SGD, called an inverse temperature parameter. 

Second, we select the momentum coefficient $\alpha_k$ for each $k=1,...,K$ as
\begin{align}
    \alpha_k=2\frac{\langle \nabla \hat{L}_{\textbf{z}}^{(k)}(\theta_k),v_{k-1}\rangle}{\|\nabla \hat{L}_{\textbf{z}}^{(k)}(\theta_k)\|^2} + C_{\alpha}, \label{def:alpha_poisson}
\end{align}
where $C_{\alpha}\ge 0$ is the hyperparameter. While $C_{\alpha}$ has the function of enhancing the effect of the gradient for practical use, we set $C_{\alpha}=0$ in the theoretical analysis of this paper (in experiments in Section \ref{section: experiments}, we set $C_{\alpha}>0$). 
This setup keeps the length of the velocity vector constant as $\|v_k\| = 1$ for every $k$ (See Proposition \ref{lem:V_norm} in Appendix), and only uses its angle to update the parameters.
We update the parameter by changing $\eta_k$ and $\alpha_k$ in every iteration.
In updating $\theta_k$, we consider modulo $a$, which means calculating modulo $a$ for every element of the vector, in order to restrict the parameter space to a torus $\mathbb{T}=(\mathbb{R}/a\mathbb{Z})^d$.
The pseudo-code of Poisson SGD is shown in Algorithm \ref{algorithm: Poisson-SGD}.

\begin{figure}[htbp]
\begin{algorithm}[H]
\caption{Poisson SGD}\label{algorithm: Poisson-SGD}
\begin{algorithmic}[1]
\STATE Initialize $(\theta_0, v_0)$ as $\|v_0\|=1$.
\FOR{$k=1,2,...,K$}
\STATE Sample $I^{(k)} \subset [n]$ and obtain $\nabla \hat{L}^{(k)}_{\textbf{z}}(\theta_k)$ as \eqref{def:batch_loss}.
\STATE Sample $\eta_k$ as \eqref{update_eta}. 
\STATE Obtain $\theta_k$ as $\theta_k=(\theta_{k-1}+\eta_kv_{k-1})\,\mathrm{mod}\,a$.
\STATE Obtain $v_k$ as $v_k=v_{k-1}-\alpha_k\nabla \hat{L}_{\textbf{z}}^{(k)}(\theta_k)$ 
with $\alpha_k$ as \eqref{def:alpha_poisson}. 
\ENDFOR
\STATE Return $(\theta_K,v_K)$.
\end{algorithmic}
\end{algorithm}
\end{figure}

The algorithm is designed to effectively explore large regions of the parameter space $\Theta$. 
Specifically, the update direction is determined by the velocity vector $v_k$ normalized by $\alpha_k$ as \eqref{def:alpha_poisson}, and the size of the update is randomly set by the random learning rate $\eta_k$ as \eqref{update_eta}.
When the gradient $\nabla \hat{L}_{\textbf{z}}^{(k)}(\cdot)$ is small, the learning rate $\eta_k$ is chosen to be large, thus the updated parameter tends to escape local minima or saddle points. 
Figure \ref{fig: poisson_sgd} illustrates that Poisson GD, which we refer to as the full-batch version of Poisson SGD, explores a wider parameter space and discovers the global minimum owing to the random learning rate, while the parameters updated by GD converge to the local minimum. Here, we set the learning rate of GD as $\eta=0.02$ and the hyperparameter of Poisson GD as $C_P=100$ and $\beta=10000$.
We remark that a small amount of hyperparameter tuning was necessary to obtain this result.

\begin{figure}[htbp]
\begin{center}
\begin{minipage}{0.5\textwidth}
\includegraphics[width=80mm]{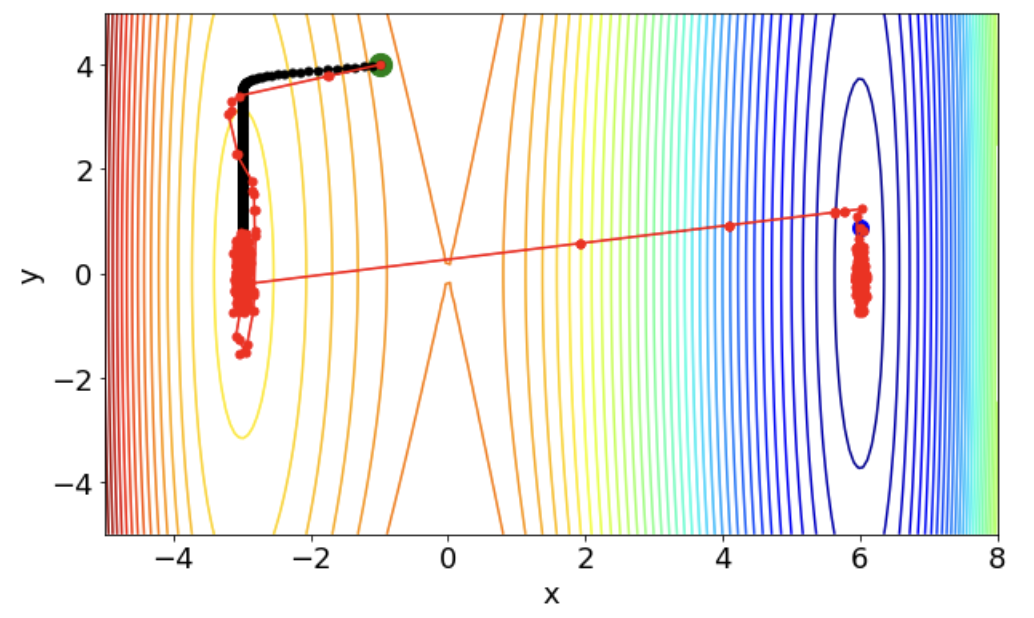}
\end{minipage}
\begin{minipage}{0.45\textwidth}
\caption{The comparison of the trajectories of GD (with a fixed learning rate) and Poisson GD (with the random learning rate) in optimizing the function $z=x^4-4x^3-36x^2+y^2$. Poisson GD represents replacing the mini-batch loss $\hat{L}_{\textbf{z}}$ by the full-batch loss $L_{\textbf{z}}$ in Poisson SGD, where we set $L_{\textbf{z}}(x,y)=x^4-4x^3-36x^2+y^2$. The point $(-3,0)$ represents a local minimum and the point $(6,0)$ is identified as the global minimum. A green point indicates the initial position, a black line represents the trajectory of GD, and a red line represents the trajectory of Poisson GD.\label{fig: poisson_sgd}}
\end{minipage}

\end{center}
\end{figure}

\begin{remark}[Moments of Poisson SGD]
We claim that even if the learning rate is random, the actual updates are not too large, by studying its moments.
That is, if $C_P$ is sufficiently large, there is little chance of sampling a large learning rate $\eta$, since the first and second moments of $\eta_k$ are given as
    $\mathbb{E}[\eta_k]
    \le \int_{0}^{\infty} \exp(-C_P t)dt=\frac{1}{C_P}$ and $ \mathbb{E}[\eta_k^2]-\mathbb{E}[\eta_k]^2
    \le \int_{0}^{\infty} 2s\exp(-C_P s)ds-\frac{1}{C_P^2}=\frac{1}{C_P^2}$.
By this property, we can avoid the case in which $\eta_k$ diverges.
In addition, even if a large $\eta_k$ is sampled, the parameter does not exit from the parameter space since we consider a torus as the parameter space.
\end{remark}

\section{Convergence Theory for Poisson SGD}

We provide theoretical results on the convergence of Poisson SGD (Algorithm \ref{algorithm: Poisson-SGD}).
Our main interest is a distribution of the generated parameter $\theta_K$ by Poisson SGD associated with the empirical risk minimization problem.

\subsection{Stationary Distribution of Poisson SGD}

In this section, we show that the parameter $\theta_K$ by the Poisson SGD follows a stationary distribution.
Formally, we define the stationary distribution of the Markov process.
In preparation, we utilize the notion of \textit{transition probability} $Q(\theta,dw)$ from a distribution $p_0(\theta)$ to another $p_1(\theta)$ on $\Theta$, that is, $p_1(w)=\int_{\Theta} Q(\theta,dw)p_0(d\theta)$ holds.
\begin{definition}[Stationary distribution]
Let $Q(\theta,dw)$ be the transition probability of a Markov process in $\Theta$. 
If the following equation holds, we call the probability distribution $\pi(\theta)$ a stationary distribution of the Markov process:
\begin{equation*}
    \pi(dw)=\int_{\Theta} Q(\theta,dw)\pi(d\theta).
\end{equation*}
\end{definition}
A stationary distribution is a useful notion to represent a limit of the parameter distribution, and it enables us to analyze where the parameter converges by algorithms. 
For example, see the theoretical framework to analyze stochastic optimization algorithms by \cite{Raginsky-2017}.

\subsubsection{Assumption}\label{section: assumption}

We provide several principal assumptions.
First, we consider the basis assumptions on the loss function $\ell(\cdot; \cdot)$.
The following conditions are fairly general for the analysis of stochastic optimization algorithms, e.g. \cite{Approx}.
\begin{assumption}[Loss function]\label{assumption: loss-function}
    The loss function $\ell: \mZ \times \Theta \to \R_{\geq 0}$ satisfies the following conditions:
    \begin{itemize}
  \setlength{\itemsep}{0mm} 
  \setlength{\parskip}{0mm} 
        \item $\ell(z; \theta)$ is absolutely continuous and differentiable with respect to $\theta \in\Theta$ for every $z \in \mZ$.
        \item $\nabla_\theta \ell(z; \theta)$ is continuous in $\theta$ and $z$ for all $\theta \in \mA$ and $z\in\mathcal{Z}$.
    \end{itemize}
\end{assumption}
These conditions are satisfied by a large class of models, such as a linear regression model or deep neural networks whose activation function is a sigmoid function or Tanh. 
We define an upper bound  $M_{\ell}:=\max_{\theta\in\Theta,z\in\mathcal{Z}}\|\nabla_{\theta}\ell(z;\theta)\|$, which is justified from the second condition on the compactness of $\Theta$ and $\mZ$.

\subsubsection{Statement of Convergence}

Let $\mu_{\textbf{z},K}$ be a distribution of the output $\theta_K$ from the Poisson SGD in Algorithm \ref{algorithm: Poisson-SGD} with the given dataset $\textbf{z}$.
We discuss the convergence of $\mu_{\textbf{z},K}$ as $K$ increases.

In preparation, we define a probability measure on $\Theta$ for arbitrary $\beta,\varepsilon>0$, whose density is written as
\begin{equation}
\label{stationary_marginal}
    \mu_{\textbf{z}}^{(\beta, \varepsilon)}(d\theta)\propto \left(\beta M_{\ell}+\frac{1}{\varepsilon}+a_d\beta\|\nabla L_{\textbf{z}}(\theta)\|\right)\exp(-\beta L_{\textbf{z}}(\theta))d\theta,
\end{equation}
where $a_d:={\Gamma(d/2)} / ({\sqrt{\pi}\Gamma(d/2+1/2)})$ with $\Gamma(z)=\int_{0}^{\infty}t^{z-1}e^{-t}dt$.
The probability measure \eqref{stationary_marginal} is concentrated around the global minima of $L_{\textbf{z}}(\theta)$, since the dominant exponential term $\exp(-\beta L_{\textbf{z}}(\theta))$ in \eqref{stationary_marginal} increases in $L_{\textbf{z}}(\theta)$.
In addition, as the inverse temperature parameter $\beta$ increases, the measure  $ \mu_{\textbf{z}}^{(\beta, \varepsilon)}$ concentrates around the global minimum.

We show our results on the convergence of the stationary distribution.
The discrepancy is measured by the Wasserstein distance $\mathcal{W}_1(\cdot, \cdot)$.
We remark that this theorem is the integration of Theorem \ref{SGD-GD} and Theorem \ref{thm: stationary-distribution-of-BPS} appearing later in Section \ref{section: proof-outline}.
Recall that we defined $W := \mathrm{diam}(\Theta)$.

\begin{theorem}[Stationary distribution of Poisson SGD]
\label{thm:stationary_Poisson_SGD}
Fix arbitrary $\beta,\varepsilon>0$. 
Suppose Assumption \ref{assumption: loss-function} holds. 
We set the $C_P ={1}/{\varepsilon}$.
Then, for any $K \in \N$, there exists $\kappa(\beta, \varepsilon,d) \in (0,1)$ such that we have 
\begin{align}
    \mathcal{W}_1(\mu_{\textbf{z},K},\mu_{\textbf{z}}^{(\beta,\varepsilon)})\leq 4\sqrt{d}K\varepsilon+ W \cdot\kappa(\beta, \varepsilon,d)^K. \label{ineq:thm2}
\end{align}

Moreover, if $\kappa(\beta, \varepsilon,d)$ satisfies $\lim_{K\rightarrow \infty}\kappa(\beta,\delta / K,d)^K=0$ with some $\delta > 0$, there exists a sequence $\varepsilon = \varepsilon_K \searrow 0$ as $K \to \infty$ such that $\mathcal{W}_1(\mu_{\textbf{z},K},\mu_{\textbf{z}}^{(\beta,\varepsilon)}) = o(1)$ as $K \to \infty$ holds.
\end{theorem}

This theorem shows that the parameter distribution $\mu_{\textbf{z},K}$ by Poisson SGD converges to the stationary distribution $\mu_{\textbf{z}}^{(\beta,\varepsilon)}$ owing to the random learning rate \eqref{update_eta}. 
This is contrast to ordinary SGD, which is not shown to converge to a stationary distribution. 
Further, Poisson SGD does not make any assumptions on the gradient noise $\xi_k^{(n,m)}$ in Remark \ref{remark:grad_noise}, unlike SGLD, which converges to a stationary distribution by introducing Gaussianity in the gradient noise.

The right-hand side in \eqref{ineq:thm2} shows an approximation-complexity trade-off of Poisson SGD described as follows.
In preparation, we will introduce a certain stochastic process to achieve the stationary distribution $\mu_{\textbf{z}}^{(\beta,\varepsilon)}$ (detail is in Section \ref{section: proof-outline}).
The first term of \eqref{ineq:thm2} describes an approximation error of Poisson SGD to the stochastic process.
The second term of \eqref{ineq:thm2} denotes a convergence error of the stochastic process to the stationary distribution $\mu_{\textbf{z}}^{(\beta,\varepsilon)}$, which reflects the complexity of the stochastic process.
$\varepsilon$ is a parameter for the stochastic process and controls the balance between the approximation error and the complexity error.

We further discuss the additional assumption $\lim_{K\rightarrow \infty}\kappa(\beta,\delta/K,d)^K=0$. This condition is related to the convergence rate of the approximated stochastic process of Poisson SGD. 
Although the explicit form of $\kappa(\beta,\delta/K,d)$ is not clarified in our case, there is a common example having its explicit form.
One example is SGLD: 
\cite{Raginsky-2017} shows that a form of $\kappa(\beta,\delta/K,d)$ can be calculated, because SGLD is reduced to the Langevin process. 

\begin{remark}[Form of $\kappa(\beta,\varepsilon,d)$]
We discuss a form of $\kappa(\beta,\varepsilon,d)$ of other related algorithms, although we could not achieve the explicit form of $\kappa(\beta,\varepsilon,d)$ of Poisson SGD. 
In the case of Langevin dynamics with the setting of \cite{Raginsky-2017}, $\kappa(\beta,\varepsilon,d)$ is $\Omega(c_{LS} k\eta / \beta(\beta+d))$, where $c_{LS}$ is the logarithmic Sobolev constant. 
On the other hand, explicitly deriving $\kappa(\beta,\varepsilon,d)$ for a class of PDMPs is a challenging task as described in \cite{Deligiannidis2017, Durmus-et-al-2020}, as well as that of Poisson-SGD.
\end{remark}

\begin{remark}[Comparison with SGLD]
    We discuss the difference between Poisson SGD and  SGLD, which is another method achieving a stationary distribution. First, while SGLD adds Gaussian noise to the update formula of SGD, Poisson SGD does not have additive noise. 
    The second difference is the form of the stationary distribution. A stationary distribution of SGLD is the Gibbs distribution, and that of Poisson SGD has a different form \eqref{stationary_marginal}. This difference is derived from the random learning rate of Poisson SGD.
\end{remark}

\begin{remark}[Relation to flat minima]
From Theorem \ref{thm:stationary_Poisson_SGD}, we can state the property of Poisson SGD being easier to go to the flat minima than the sharp minima. 
We consider the probability of existence around a flat minimum $\theta_1 \in \Theta$ and a sharp minimum $\theta_2 \in \Theta$, when we find that, due to the shape of the distribution, a measure of an $\epsilon$-neighborhood of $\theta_1$ is greater than that within an $\epsilon$-neighborhood of $\theta_2$. Hence, we can claim that Poisson SGD also tends to favor flat minima.

\end{remark}

\subsection{Global Convergence}

We discuss the global convergence statement, that is, the empirical risk $L_{\textbf{z}}(\theta_K)$ with Poisson SGD is minimized with high probability.
We consider the additional assumption for the loss function $\ell$:
\begin{assumption}
\label{assumption: loss2}
    With some $c_1 > 0$, $\sup_{z \in \mZ}\|\nabla \ell(z; \theta_1)-\nabla \ell(z; \theta_2)\|\leq c_1 \|\theta_1-\theta_2\|$ holds for every $\theta_1 \neq \theta_2 \in \mA$.
\end{assumption}

Then, we obtain the following global convergence theorem. 
\begin{theorem}[Global convergence of Poisson SGD on empirical risk]
\label{thm: global_convergence}
    Fix arbitrary $\beta,\varepsilon>0$.
    Consider Poisson SGD in which $C_P={1} / {\varepsilon}$.
    Let the upper bound of $\mathcal{W}_1(\mu_{\textbf{z},K},\mu_{\textbf{z}}^{(\beta,\varepsilon)})$ obtained in Theorem \ref{thm:stationary_Poisson_SGD} be $d_K(\beta,\varepsilon,d)$.
    Suppose Assumption \ref{assumption: loss-function} and \ref{assumption: loss2} hold,  and define $B:=\sup_{z \in \mZ}\|\nabla \ell(z;0)\|$ and $c_2 = c_1W+B$. Then, it holds that
    \begin{align}
        &\mathbb{E}_{\theta_K \sim \mu_{\textbf{z},K}}[L_{\textbf{z}}(\theta_K)]-\min_{\theta\in\Theta}L_{\textbf{z}}(\theta) 
        \leq c_2\sqrt{Wd_K(\beta,\varepsilon,d)}+\frac{1}{\beta}\left(\frac{d}{2}\log \frac{eW^2c_1\beta}{d}+\log \left(1+\frac{a_d c_2}{M_{\ell}}\right)\right). \label{ineq:thm4}
    \end{align}
\end{theorem}

Theorem \ref{thm: global_convergence} states that we can make $\Ep[L_{\textbf{z}}(\theta_K)]$ be arbitrarily close to $\min_{\theta\in\Theta}L_{\textbf{z}}(\theta)$ by selecting large $\beta$, provided that we can make $d_K(\beta,\varepsilon,d)$ arbitrarily small by the choice of $\varepsilon$ and $K$ in spite of $\beta$.
Intuitively, Poisson SGD achieves global convergence by appropriately adjusting the learning rate and momentum coefficient based on the shape of the loss function at the current location. 
Poisson SGD achieves the global convergence by the similar approach of global convergence of SGLD by \cite{Raginsky-2017}. 

The right-hand side of \eqref{ineq:thm4} is divided into two terms. The first term expresses the distance between the parameter and its stationary distribution. 
The second represents the degree of concentration of the stationary distribution $\mu_{\textbf{z}}^{(\beta,\varepsilon)}$ on the global optima. 
The higher the inverse temperature $\beta$, the more the term decreases.

\section{Proof Outline} \label{section: proof-outline}
\subsection{Overview}
We give an overview of a proof of Theorem \ref{thm:stationary_Poisson_SGD}.
In preparation, we present several key concepts: (i) the property of the \textit{piece-wise deterministic Markov process} (PDMP) \citep{PDMP1,PDMP2}, and (ii) the ergodicity of \textit{bouncy particle sampler} (BPS) \citep{BPS-first}.
The PDMP is a class of Markov processes that behave deterministically for some period and jump randomly, which easily converges to a stationary distribution.
BPS is a stochastic algorithm in the class of the PDMP. 
This BPS is virtually constructed to prove the convergence of PSGD and is not actually computed in this study. Therefore, there is no need to consider the computational cost of this BPS.

We show the statement by the following steps:
    (I) We show that the distribution of the parameter by  Poisson SGD is sufficiently close to that of a parameter by BPS.
    We show this claim by using the approximation theory on PDMP (Theorem \ref{SGD-GD}), and (II)
    We derive a stationary distribution and the ergodicity of BPS, following previous researches (Theorem \ref{thm: stationary-distribution-of-BPS}).

\subsection{Design of BPS}
We introduce BPS, which is one of the most popular algorithms in PDMPs, and actively studied in terms of MCMC algorithm \citep{Deligiannidis2017, rejection-free}.
BPS generates a sequence of parameters $\{\hat{\theta}_k\}_{k=1}^K \subset \Theta$ and velocity vectors $\{\hat{v}_k\}_{k=1}^K \subset \R^d$ in its recursive manner, as shown in Algorithm \ref{algorithm: BPS}.
Let $(\hat{\theta}_0, \hat{v}_0)$ be the initialization.
For the $k$-th iteration, BPS generates a learning rate $\hat{\eta}_k$ from an exponential distribution whose intensity depends on the previous pair $(\hat{\theta}_{k-1}, \hat{v}_{k-1})$ and the positive constants $\Lambda_{\mathrm{ref}}$ and $C_B$. 
In the same way as Poisson SGD, we calculate modulo $a$ when updating $\hat{\theta}_k$ for restricting the parameter space to a torus.
After obtaining the parameter $\hat{\theta}_k$, we consider the stochastic update of the velocity vector.
That is, with the probability 
\begin{align}
    \hat{p}_k := \frac{\beta\langle \nabla L_{\textbf{z}}(\hat{\theta}_k),\hat{v}_{k-1}\rangle_++C_B}{\beta\langle \nabla L_{\textbf{z}}(\hat{\theta}_k),\hat{v}_{k-1}\rangle_++\Lambda_{\mathrm{ref}}+C_B}, \label{def:prob_bps}
\end{align}
we update the velocity vector with the gradient of the full-batch loss $\nabla L_{\textbf{z}}$, otherwise with the sample from the uniform distribution on $\mathbb{S}^{d-1}$. The former update is called {\em reflection}, and the latter is {\em refreshment}.
We remark that $\|\hat{v}_k\|$ is constant for $k=1,2,...,K$ in the same way as Poisson SGD (See Proposition \ref{lem:V_norm} in Appendix).

\begin{algorithm}[tb]
\caption{Bouncy Particle Sampler}
\label{algorithm: BPS}
\begin{algorithmic}[1]
\STATE Initialize $(\hat{\theta}_0, \hat{v}_0)$ as $\|\hat{v}_0\|=1$.
\FOR{$k=1,2,...,K$}
\STATE Sample $\hat{\eta}_k$ as $\hat{\eta}_k\sim P(\hat{\eta}_k\geq t)=\exp\left(-\int_{0}^{t}\{\beta\langle \nabla L_{\textbf{z}}(\hat{\theta}_{k-1}+r\hat{v}_{k-1}),\hat{v}_{k-1}\rangle_++\Lambda_{\mathrm{ref}}+C_B\}dr\right)$
\STATE Update $\hat{\theta}_k$ as $\hat{\theta}_k=(\hat{\theta}_{k-1}+\hat{\eta}_k\hat{v}_{k-1})\,\mathrm{mod}\,a$
\STATE With probability $\hat{p}_k$ as \eqref{def:prob_bps}, update $\hat{v}_k$ as 
\begin{equation*}
\hat{v}_k=\hat{v}_{k-1}-2\frac{\langle \nabla L_{\textbf{z}}(\hat{\theta}_k),\hat{v}_{k-1}\rangle}{\|\nabla L_{\textbf{z}}(\hat{\theta}_k)\|^2}\nabla L_{\textbf{z}}(\hat{\theta}_k)
\end{equation*}
Otherwise, update $\hat{v}_k$ as $\hat{v}_k \sim \mathrm{Unif}(\mathbb{S}^{d-1})$
\ENDFOR
\STATE Return $(\hat{\theta}_K, \hat{v}_K)$
\end{algorithmic}
\end{algorithm}

\subsection{Connect Poisson SGD and BPS}
We show that the output distribution of Poisson SGD and that of BPS are sufficiently close as follows:

\begin{theorem}[Distance between Poisson SGD and BPS]
\label{SGD-GD}
Fix arbitrary $\beta,\varepsilon>0$. As for Poisson SGD, we set $C_P={1} / {\varepsilon}$. As for BPS, we set $\Lambda_{\mathrm{ref}}$ and $C_B$ as $\Lambda_\mathrm{ref}+C_B=\beta M_{\ell}+{1} / {\varepsilon}$. Let the distribution of the obtained parameter by Poisson SGD and BPS be $\mu_{\textbf{z},K}$ and $\hat{\mu}_{\textbf{z},K}$ respectively. We set the same initial value between Poisson SGD and BPS. Then, the following holds:
    \begin{equation}
    \label{wasserstein2}
        \mathcal{W}_1(\mu_{\textbf{z},K},\hat{\mu}_{\textbf{z},K})\leq 4\sqrt{d}K\varepsilon.
    \end{equation}
\end{theorem}
For proving this theorem, we calculate the distance between Poisson SGD and BPS by a one-step update. Then, we simply accumulate this error for $K$ times.
In this discussion, we mainly use the property that if learning rate $\eta_k$ and $\hat{\eta}_k$ are small, the difference of $v_k$ and $\hat{v}_k$ is also made to be small. This type of discussion is also used in \cite{Raginsky-2017}.

\subsection{The Stationary Distribution and Ergodicity of BPS}

We investigate the stationary distribution and ergodicity of BPS.
First, we define the term \textit{ergodicity}. 
\begin{definition}[Ergodicity]
We consider the discrete-time Markov process. If the process converges to a unique stationary distribution, we call the process has the ergodicity.
Especially, if the ergodic process converges to its stationary distribution by the exponential rate about the number of iteration $k$, the process is called exponentially ergodic.
\end{definition}

Without ergodicity, the stochastic process may converge to more than one stationary distribution, or not converge to any stationary distribution due to stacking to a saddle point in the parameter space. So we have to prove this property when we try to analyze the stationary distribution of a stochastic process.

Now, we show our result about BPS.
\begin{theorem}[Stationary Distribution of BPS]
\label{thm: stationary-distribution-of-BPS}
Suppose that Assumption \ref{assumption: loss-function} holds.
Set the parameter of BPS, $\Lambda_{\mathrm{ref}}$ and $C_B$ as in Theorem \ref{SGD-GD}.
Then, the distribution $\hat{\mu}_{\textbf{z},K}$ of the obtained parameters $\hat{\theta}_K$ by BPS satisfies the following inequality with some constant $\kappa(\beta, \varepsilon,d) \in (0,1)$:
\begin{align}
    \|\hat{\mu}_{\textbf{z},K}- \mu_{\textbf{z}}^{(\beta, \varepsilon)}\|_{\mathrm{TV}} \leq \kappa(\beta, \varepsilon,d)^K.
\end{align}
\end{theorem}

In its proof, we use the discussion in \cite{Deligiannidis2017} which showed that continuous-time BPS converges to the unique stationary distribution $\pi(\theta)\propto \exp(-U(\theta))$ by the exponential rate in TV distance.

\section{Generalization Error Analysis}

We define an expected risk of $\theta \in \mA$, also known as the generalization error
$L(\theta) := \mathbb{E}_{z\sim P_*}[\ell(z;\theta)]$,
which measures a prediction performance with unseen data.
We calculate the generalization error of the parameter obtained by the Poisson SGD, using the discussion in \cite{Raginsky-2017}. 

Now, we give our results. 
We define $A:=\sup_{z \in \mZ}| \ell(z;0)|$ by following Assumption \ref{assumption: loss-function}.
\begin{theorem}[Generalization Error of Poisson SGD]
\label{thm: generalization_error}
    Suppose that Assumption \ref{assumption: loss-function} and \ref{assumption: loss2} hold. 
    Let $\theta_K$ be the parameter obtained by Poisson SGD with $C_P={1} / {\varepsilon}$.
    Then, we obtain the following bound:
    \begin{align}
        &\mathbb{E}_{\textbf{z}\sim P_*^n}[\mathbb{E}_{\theta_K\sim \mu_{\textbf{z},K}}[L(\theta_K)]]-\min_{\theta\in \Theta}L(\theta) \\
        & \leq c_2 \left(\sqrt{Wd_K(\beta,\varepsilon,d)} + 2W\left(\left(\frac{C'_d}{n}\right)^{{1}/  {2}}+\left(\frac{C'_d}{n}\right)^{{1} / {4}}\right)\right) 
        +\frac{1}{\beta}\left(\frac{d}{2}\log \frac{eW^2c_1\beta}{d}+\log \left(1+\frac{a_d C_d}{M_{\ell}}\right)\right),
    \end{align}
    where $d_K(\beta,\varepsilon,d)$ is the upper bound of the Wasserstein distance in Theorem \ref{thm:stationary_Poisson_SGD}, $C_d={4a_d(c_1W+B)} / {M_{\ell}}$, $C'_d = C_d+\beta C$ and $C=c_1W^2+2BW+2A$.
\end{theorem}

Theorem \ref{thm: generalization_error} states that the expected value of the generalization error of Poisson SGD can be arbitrarily close to its global optima in $\theta\in\Theta$, by selecting small $\varepsilon$, large $K$, large $\beta$, and large $n$, provided that $d_K(\beta,\varepsilon,d)$ can be arbitrarily small only by the choice of $\varepsilon$ and $K$.

We further discuss a way of improve an order of the generalization bound in Theorem \ref{thm: generalization_error}.
While our bound has the order $O((1/n)^{1/4})$, we can obtain an order $O(1/n)$ by using the \textit{dissipativity} condition of the loss function, which is used in \cite{Raginsky-2017} for SGLD.
The dissipativity condition allows us to derive log-Sobolev inequality for $L_{\textbf{z}}(\theta)$, which leads the improved sample complexity. 
We state this fact in the following proposition.
\begin{proposition}\label{prop: log-Sobolev}
    Suppose that the same condition and setting as Theorem \ref{thm: generalization_error} hold. In addition, we assume that the Gibbs distribution $\nu_{\textbf{z}}^{(\beta)}\propto\exp(-\beta L_{\textbf{z}}(\theta))$ satisfies the log-Sobolev inequality for any dataset $\textbf{z}=\{z_1,...,z_n\}$, that is, 
        $\mathbb{E}[f(\theta)^2\log f(\theta)^2]-\mathbb{E}[f(\theta)^2]\log\mathbb{E}[f(\theta)^2]\leq c_\mathrm{LS}^{(\beta)}\mathbb{E}[\|\nabla f(\theta)\|^2]$
    holds for all smooth functions $f$ and any data $\textbf{z}=\{z_1,...,z_n\}$, where $\theta\sim\nu_{\textbf{z}}^{(\beta)}$ and $c_\mathrm{LS}^{(\beta)}<\infty$ is a constant.
    Then, the following holds:
    \begin{align}
        &\mathbb{E}_{\textbf{z}\sim P_*^n}[\mathbb{E}_{\theta_K\sim\mu_{\textbf{z},K}}[L(\theta_K)]]-\min_{\theta\in \Theta}L(\theta) \\
        &\leq c_2\left(\sqrt{Wd_K(\beta,\varepsilon,d)} + \frac{2c_\mathrm{LS}^{(\beta)}\beta M_{\ell}}{n}\right)
        +W\sqrt{2c_\mathrm{LS}^{(\beta)}\log\left(1+a_d\beta \varepsilon M_{\ell}\right)}+\frac{d}{2\beta}\log\left(\frac{eW^2c_1\beta}{d}\right).
    \end{align}
\end{proposition}

\section{Experiments}\label{section: experiments}
We give several experimental results to validate our theoretical claim.
For all experiments, the detailed setup such as the choice of the hyperparameter is described in Section \ref{sec:detail_experiment}.

\subsection{Convergence to Stationary Distribution} \label{sec:exp_stationary}
We consider a distribution of parameters updated by Poisson SGD, then experimentally validate whether the distribution converges to the derived stationary distribution as the number of updates increases. 
Specifically, we numerically study whether the distribution $\mu_{\textbf{z},K}$ of parameters generated by Poisson SGD converges to the theoretical stationary distribution $\mu_{\textbf{z}}^{(\beta, \varepsilon)}$ we have derived.

We describe an outline of the setup.
We set $d = 2$ and consider a parameter $\theta = (\theta_1, \theta_2) $ in the $2$-dimensional torus $\Theta = (\mathbb{R}/a\mathbb{Z})^2$ with $a = 20$. 
The loss function is set as $\ell(z;\theta) = ((\theta_1^2-\theta_2^2)x_1^2+(\theta_1^2-\theta_2^2)x_2^2 - y)^2$, then the stationary distribution $\mu_{\textbf{z}}^{(\beta, \varepsilon)}$ is well approximated by $\exp(-\beta'(\theta_1^2-\theta_2^2-1)^2)$, where $\beta'$ is a constant multiple of $\beta$.
Hence, the stationary distribution $\mu_{\textbf{z}}^{(\beta, \varepsilon)}$ concentrates around a set  $\{(\theta_1,\theta_2) \in \Theta \mid \theta_1^2-\theta_2^2-1=0\}$. 
We obtain the numerical distribution $\mu_{\textbf{z},K}$ of the parameter $\theta$ by Poisson SGD for $10,100,$ and $1000$ epoch.

Figure \ref{fig: 2d_experiment} shows the distribution $\mu_{\textbf{z},K}$ on the parameter space $\Theta$. We remark that the value of the parameter is in $[0, 20]$ since the parameter space is the torus $(\mathbb{R}/a\mathbb{Z})^2$, and the tick starts from $10$ instead of $0$ for the sake of clarity. The contour shows the value of $(\theta_1^2-\theta_2^2-1)^2$.
As the epoch increases, the parameter sampled from Poisson SGD is distributed around a set $\{(\theta_1,\theta_2) \in \Theta \mid \theta_1^2-\theta_2^2-1=0\}$. This shows that the parameter distribution by Poisson SGD converges to the stationary distribution $\mu_{\textbf{z}}^{(\beta, \varepsilon)}$.

\begin{figure}[htbp]
\begin{center}
\includegraphics[width=160mm]{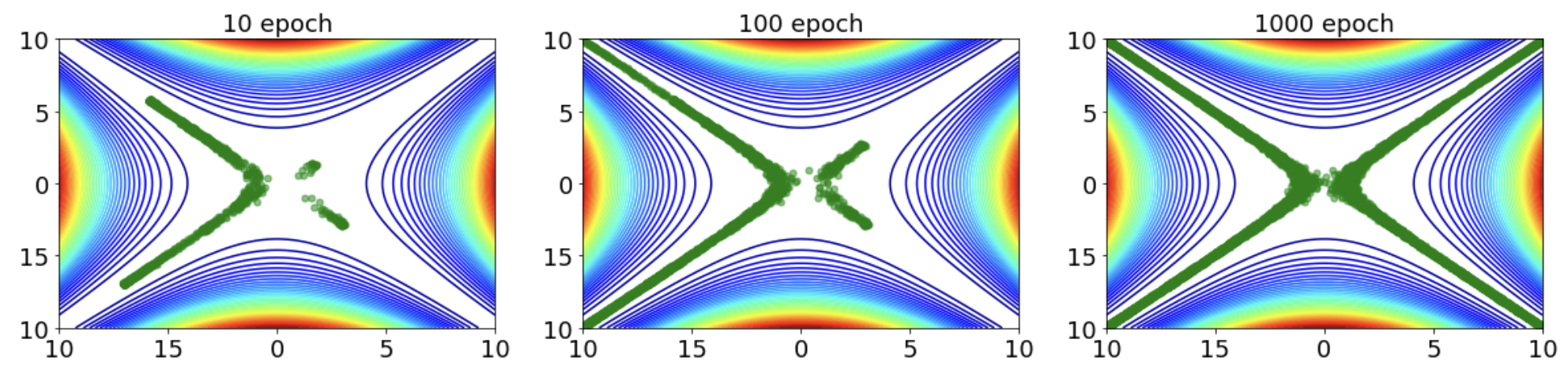}
\caption{Sample from Poisson SGD for 10 (left), 100 (middle), and 1000 (right) epochs.}\label{fig: 2d_experiment}
\end{center}
\end{figure}

Also, Figure \ref{fig: 2d_experiment_wd} shows the Wasserstein distance between the stationary distribution $\mu_{\textbf{z}}^{(\beta, \varepsilon)}$ and the distribution of parameters learned by SGD or Poisson SGD. 
The horizontal axis represents the epochs and the vertical axis shows the Wasserstein distance. 
While the Wasserstein distance does not converge in the SGD case, it converges toward zero in the Poisson SGD case.
This result indicates that the parameter learned by Poisson SGD has the distribution $\mu_{\textbf{z},K}$ which converges to the stationary distribution $\mu_{\textbf{z}}^{(\beta, \varepsilon)}$. 

\begin{wrapfigure}[9]{r}[0pt]{0.4\textwidth}
\vspace{-50pt}
\begin{center}
\includegraphics[width=0.99\hsize]{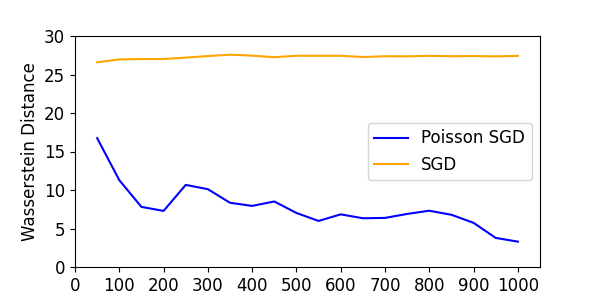}
\caption{Wasserstein distance from the distribution of parameters updated by SGD (orange) and Poisson SGD (blue) to $\mu_{\textbf{z}}^{(\beta, \varepsilon)}$.}\label{fig: 2d_experiment_wd}
\end{center}
\vspace{-50pt}
\end{wrapfigure}

\subsection{Optimization and Generalization Performance} \label{sec:exp_opt}
We verify the training and generalization performance of Poisson SGD in a practical situation. 
Note that our aim is not to develop an effective method with high generalization performance, but to develop a method that can guarantee global convergence.

We conducted experiments with several datasets and models. First is the MNIST dataset \citep{deng2012mnist} on a fully connected neural network of $4$ layers, and each layer has $200$ units and the sigmoid activation function. Second is the CIFAR-10 dataset \citep{krizhevsky2009learning} on a convolutional neural network of $3$ layers with the ReLU activation function, a $3\times 3$ kernel, and the dropout rate $0.25$. Third is the CIFAR-100 dataset \citep{krizhevsky2009learning} on VGG16 \citep{simonyan2015verydeep}. And the last is the CIFAR-100 dataset on ResNet18 \citep{he2015resnet}. We compared the performance of Poisson SGD with several existing optimization methods.

We display the overall result in Figure \ref{fig:all_experiments}. In all the cases, the performance of Poisson SGD is better or comparable to that of other methods.

\begin{figure}
\noindent
    \centering
    \includegraphics[width=0.8\linewidth]{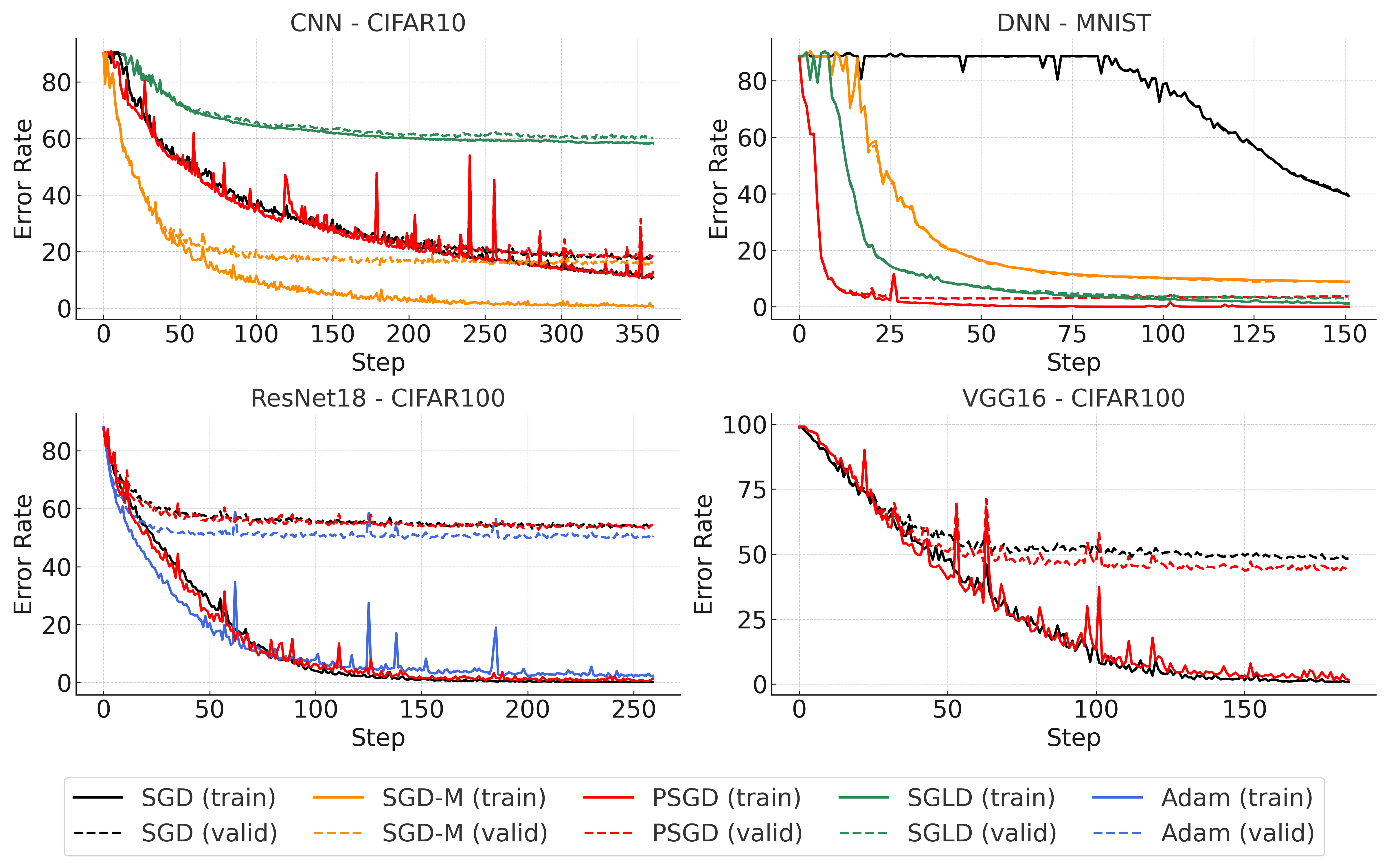}
    \caption{ Error rates against the update steps for each dataset/network.}
    \label{fig:all_experiments}
    \vskip-0.2in
\end{figure}

\section{Conclusion}
We developed a new variant of SGD, Poisson SGD, whose search direction degenerates and derived its stationary distribution by incorporating a modification on the learning rate.
The parameters trained by Poisson SGD are close enough to the global minima to take advantage of convergence to the stationary distribution. 
The generalization error is also evaluated.
Our study suggests that even local search methods such as SGD may be able to achieve global convergence by adding certain randomness.
Since our work largely depends on the modification of the learning rate, eliminating this dependence is an interesting open question.
Hence, the dynamics of constant learning rate SGD should be clarified as a future direction.

\bibliography{main}

\begin{thebibliography}{47}
\providecommand{\natexlab}[1]{#1}
\providecommand{\url}[1]{\texttt{#1}}
\expandafter\ifx\csname urlstyle\endcsname\relax
  \providecommand{\doi}[1]{doi: #1}\else
  \providecommand{\doi}{doi: \begingroup \urlstyle{rm}\Url}\fi

\bibitem[Bakry et~al.(2014)Bakry, Gentil, and Ledoux]{Bakry-2014}
Dominique Bakry, Ivan Gentil, and Michel Ledoux.
\newblock \emph{Analysis and Geometry of Markov Diffusion Operators}.
\newblock Springer, 2014.

\bibitem[Battash et~al.(2024)Battash, Wolf, and Lindenbaum]{revisiting}
Barak Battash, Lior Wolf, and Ofir Lindenbaum.
\newblock Revisiting the noise model of stochastic gradient descent.
\newblock In \emph{International Conference on Artificial Intelligence and Statistics}, 2024.

\bibitem[Bertazzi et~al.(2022)Bertazzi, Bierkens, and Dobson]{Approx}
Andrea Bertazzi, Joris Bierkens, and Paul Dobson.
\newblock Approximations of piecewise deterministic markov processes and their convergence properties.
\newblock \emph{Stochastic Processes and their Applications}, 154:\penalty0 91--153, 2022.

\bibitem[Blier et~al.(2019)Blier, Wolinski, and Ollivier]{Blier}
Léonard Blier, Pierre Wolinski, and Yann Ollivier.
\newblock Learning with random learning rates.
\newblock In \emph{ECML PKDD}, 2019.

\bibitem[Bolley \& Villani(2005)Bolley and Villani]{Bolley_and_Villani}
Fran{\c{c}}ois Bolley and C{\'e}dric Villani.
\newblock Weighted csisz{\'a}r-kullback-pinsker inequalities and applications to transportation inequalities.
\newblock \emph{Annales de la Facult{\'e} des sciences de Toulouse: Math{\'e}matiques}, 14\penalty0 (3):\penalty0 331--352, 2005.

\bibitem[Bottou(1991)]{bottou1991stochastic}
L{\'e}on Bottou.
\newblock Stochastic gradient learning in neural networks.
\newblock \emph{Proceedings of Neuro-N{\i}mes}, 91\penalty0 (8):\penalty0 12, 1991.

\bibitem[Bouchard-C{\^o}t{\'e} et~al.(2018)Bouchard-C{\^o}t{\'e}, Vollmer, and Doucet]{rejection-free}
Alexandre Bouchard-C{\^o}t{\'e}, Sebastian~J Vollmer, and Arnaud Doucet.
\newblock The bouncy particle sampler: A nonreversible rejection-free markov chain monte carlo method.
\newblock \emph{Journal of the American Statistical Association}, 113\penalty0 (522):\penalty0 855--867, 2018.

\bibitem[Chen et~al.(2022)Chen, Mou, and Maguluri]{strongly-convex2}
Zaiwei Chen, Shancong Mou, and Siva~Theja Maguluri.
\newblock Stationary behavior of constant stepsize sgd type algorithms: An asymptotic characterization.
\newblock \emph{Proceedings of the ACM on Measurement and Analysis of Computing Systems}, 6(1)\penalty0 (19):\penalty0 1--24, 2022.

\bibitem[Cheng et~al.(2020)Cheng, Yin, Bartlett, and Jordan]{Langevin-SGD1}
Xiang Cheng, Dong Yin, Peter Bartlett, and Michael Jordan.
\newblock Stochastic gradient and langevin processes.
\newblock In \emph{International Conference on Machine Learning}, 2020.

\bibitem[Dalalyan(2017)]{Langevin1}
Arnak Dalalyan.
\newblock Further and stronger analogy between sampling and optimization: Langevin monte carlo and gradient descent.
\newblock In \emph{Conference on Learning Theory}, 2017.

\bibitem[Davis(1984)]{PDMP1}
M.H.A Davis.
\newblock Piecewise-deterministic markov processes: A general class of non-diffusion stochastic models.
\newblock \emph{Journal of the Royal Statistical Society}, Series B (Methodological)\penalty0 (46):\penalty0 353--388, 1984.

\bibitem[Davis(1993)]{PDMP2}
M.H.A Davis.
\newblock \emph{Markov Models and Optimization}.
\newblock Chapman \& Hall/CRC Monographs on Statistics \& Applied Probability. Taylor \& Francis, 1993.

\bibitem[Deligiannidis et~al.(2019)Deligiannidis, Bouchard-Côté, and Doucet]{Deligiannidis2017}
George Deligiannidis, Alexandre Bouchard-Côté, and Arnaud Doucet.
\newblock Exponential ergodicity of the bouncy particle sampler.
\newblock \emph{The Annals of Statistics}, 47:\penalty0 1268--1287, 2019.

\bibitem[Deng(2012)]{deng2012mnist}
Li~Deng.
\newblock The mnist database of handwritten digit images for machine learning research.
\newblock \emph{IEEE Signal Processing Magazine}, 29\penalty0 (6):\penalty0 141--142, 2012.

\bibitem[Dieuleveut et~al.(2020)Dieuleveut, Durmus, and Bach]{strongly-convex1}
Aymeric Dieuleveut, Alain Durmus, and Francis Bach.
\newblock Bridging the gap between constant step size stochastic gradient descent and markov chains.
\newblock \emph{The Annals of Statistics}, 48(3):\penalty0 1348--1382, 2020.

\bibitem[Durmus \& Moulines(2016)Durmus and Moulines]{Langevin2}
Alain Durmus and \'{E}ric Moulines.
\newblock Sampling from a strongly log-concave distribution with the unadjusted langevin algorithm.
\newblock \emph{HAL preprint hal-01304430v1}, 2016.

\bibitem[Durmus et~al.(2020)Durmus, Guillin, and Monmarché]{Durmus-et-al-2020}
Alain Durmus, Arnaud Guillin, and Pierre Monmarché.
\newblock Geometric ergodicity of the bouncy particle sampler.
\newblock \emph{Annals of Applied Probability}, 30:\penalty0 2069--2098, 2020.

\bibitem[Faggionato et~al.(2009)Faggionato, Gabrielli, and Ribezzi-Crivellari]{torus3}
Alessandra Faggionato, Davide Gabrielli, and Marco Ribezzi-Crivellari.
\newblock Non-equilibrium thermodynamics of piecewise deterministic markov processes.
\newblock \emph{Journal of Statistical Physics}, 137:\penalty0 259--304, 2009.

\bibitem[Garrigos \& Gower(2023)Garrigos and Gower]{handbook}
Guillaume Garrigos and Robert~M. Gower.
\newblock Handbook of convergence theorems for (stochastic) gradient methods.
\newblock \emph{arXiv preprint arXiv:2301.11235}, 2023.

\bibitem[Gibbs \& Su(2002)Gibbs and Su]{Wasserstein-TV}
Alison~L. Gibbs and Francis~Edward Su.
\newblock On choosing and bounding probability metrics.
\newblock \emph{International Statistical Review / Revue Internationale de Statistique}, 70\penalty0 (3):\penalty0 419--435, 2002.

\bibitem[HaoChen et~al.(2021)HaoChen, Wei, Lee, and Ma]{Langevin-SGD2}
Jeff~Z. HaoChen, Colin Wei, Jason~D. Lee, and Tengyu Ma.
\newblock Shape matters: Understanding the implicit bias of the noise covariance.
\newblock In \emph{Conference On Learning Theory}, 2021.

\bibitem[He et~al.(2019)He, Liu, and Tao]{SGD-noise-gaussian1}
Fengxiang He, Tongliang Liu, and Dacheng Tao.
\newblock Control batch size and learning rate to generalize well: Theoretical and empirical evidence.
\newblock In \emph{Advances in Neural Information Processing Systems}, 2019.

\bibitem[He et~al.(2015)He, Zhang, Ren, and Sun]{he2015resnet}
Kaiming He, Xiangyu Zhang, Shaoqing Ren, and Jian Sun.
\newblock Deep residual learning for image recognition.
\newblock \emph{arXiv preprint arXiv:1512.03385}, 2015.

\bibitem[Hoffer et~al.(2017)Hoffer, Hubara, and Soudry]{hoffer2017train}
Elad Hoffer, Itay Hubara, and Daniel Soudry.
\newblock Train longer, generalize better: closing the generalization gap in large batch training of neural networks.
\newblock \emph{Advances in neural information processing systems}, 30, 2017.

\bibitem[Jastrzebski et~al.(2017)Jastrzebski, Kenton, Arpit, Ballas, Fischer, Bengio, and Storkey]{jastrzkebski2017three}
Stanislaw Jastrzebski, Zachary Kenton, Devansh Arpit, Nicolas Ballas, Asja Fischer, Yoshua Bengio, and Amos Storkey.
\newblock Three factors influencing minima in sgd.
\newblock \emph{arXiv preprint arXiv:1711.04623}, 2017.

\bibitem[Keskar et~al.(2016)Keskar, Mudigere, Nocedal, Smelyanskiy, and Tang]{keskar2016large}
Nitish~Shirish Keskar, Dheevatsa Mudigere, Jorge Nocedal, Mikhail Smelyanskiy, and Ping Tak~Peter Tang.
\newblock On large-batch training for deep learning: Generalization gap and sharp minima.
\newblock In \emph{International Conference on Learning Representations}, 2016.

\bibitem[Krizhevsky et~al.(2009)Krizhevsky, Hinton, et~al.]{krizhevsky2009learning}
Alex Krizhevsky, Geoffrey Hinton, et~al.
\newblock Learning multiple layers of features from tiny images.
\newblock 2009.

\bibitem[Latz(2021)]{Latz}
Jonas Latz.
\newblock Analysis of stochastic gradient descent in continuous time.
\newblock \emph{Statistics and Computing}, 31(39), 2021.

\bibitem[Lee(2012)]{torus-textbook}
John~M. Lee.
\newblock \emph{Introduction to Smooth Manifolds}, volume 218 of \emph{Graduate Texts in Mathematics}.
\newblock Springer, 2012.

\bibitem[Li et~al.(2017)Li, Tai, and E]{SGD-diffusion}
Qianxiao Li, Cheng Tai, and Weinan E.
\newblock Stochastic modified equations and adaptive stochastic gradient algorithms.
\newblock In \emph{International Conference on Machine Learning}, 2017.

\bibitem[Mandt et~al.(2017)Mandt, Hoffman, and Blei]{SGD-noise-gaussian2}
Stephan Mandt, Matthew~D. Hoffman, and David~M. Blei.
\newblock Stochastic gradient descent as approximate bayesian inference.
\newblock \emph{Journal of Machine Learning Research}, 18:\penalty0 1--35, 2017.

\bibitem[Monmarché(2016)]{torus1}
Pierre Monmarché.
\newblock Piecewise deterministic simulated annealing.
\newblock \emph{Latin American Journal of Probability and Mathematical Statistics}, 13:\penalty0 357--398, 2016.

\bibitem[Musso(2020)]{Musso}
Daniele Musso.
\newblock Stochastic gradient descent with random learning rate.
\newblock \emph{arXiv preprint arXiv:2003.06926}, 2020.

\bibitem[Nguyen et~al.(2019)Nguyen, Simsekli, Gürbüzbalaban, and Richard]{nguyen2019first}
Thanh~Huy Nguyen, Umut Simsekli, Mert Gürbüzbalaban, and Gaël Richard.
\newblock First exit time analysis of stochastic gradient descent under heavy-tailed gradient noise.
\newblock \emph{Advances in neural information processing systems}, 32, 2019.

\bibitem[Nickl \& Ray(2020)Nickl and Ray]{torus2}
Richard Nickl and Kolyan Ray.
\newblock Nonparametric statistical inference for drift vector fields of multi-dimensional diffusions.
\newblock \emph{The Annals of Statistics}, 48:\penalty0 1383--1408, 2020.

\bibitem[Peters \& de~With(2012)Peters and de~With]{BPS-first}
E.~A. J.~F. Peters and G.~de~With.
\newblock Rejection-free monte carlo sampling for general potentials.
\newblock \emph{Physical Review E}, 2012.

\bibitem[Peutrec et~al.(2022)Peutrec, Michel, and Nectoux]{torus4}
Dorian~Le Peutrec, Laurent Michel, and Boris Nectoux.
\newblock Eyring-kramers type formulas for some piecewise deterministic markov processes.
\newblock \emph{arXiv preprint arXiv:2202.01465}, 2022.

\bibitem[Qi \& Luo(2013)Qi and Luo]{qi2013bounds}
Feng Qi and Qiu-Ming Luo.
\newblock Bounds for the ratio of two gamma functions: from {Wendel}’s asymptotic relation to {Elezovi\'{c}}-{Giordano}-{Pe\v{c}ari\'{c}}’s theorem.
\newblock \emph{Journal of Inequalities and Applications}, 2013\penalty0 (1):\penalty0 542, November 2013.

\bibitem[Qian(1999)]{momentum-1}
Ning Qian.
\newblock On the momentum term in gradient descent learning algorithms.
\newblock \emph{Neural Networks}, pp.\  145--151, 1999.

\bibitem[Raginsky et~al.(2017)Raginsky, Rakhlin, and Telgarsky]{Raginsky-2017}
Maxim Raginsky, Alexander Rakhlin, and Matus Telgarsky.
\newblock Non-convex learning via stochastic gradient langevin dynamics: a nonasymptotic analysis.
\newblock In \emph{Conference On Learning Theory}, 2017.

\bibitem[Sherlock \& Thiery(2022)Sherlock and Thiery]{discreteBPS}
Chris Sherlock and Alexandre~H. Thiery.
\newblock A discrete bouncy particle sampler.
\newblock \emph{Biometrika}, 109:\penalty0 335--349, 2022.

\bibitem[Simonyan \& Zisserman(2015)Simonyan and Zisserman]{simonyan2015verydeep}
Karen Simonyan and Andrew Zisserman.
\newblock Very deep convolutional networks for large-scale image recognition.
\newblock In \emph{International Conference on Learning Representations}, 2015.

\bibitem[Simsekli et~al.(2019)Simsekli, Sagun, and Gurbuzbalaban]{simsekli2019tail}
Umut Simsekli, Levent Sagun, and Mert Gurbuzbalaban.
\newblock A tail-index analysis of stochastic gradient noise in deep neural networks.
\newblock In \emph{International Conference on Machine Learning}, pp.\  5827--5837. PMLR, 2019.

\bibitem[Sutskever et~al.(2013)Sutskever, Martens, Dahl, and Hinton]{momentum-2}
Ilya Sutskever, James Martens, George Dahl, and Geoffrey Hinton.
\newblock On the importance of initialization and momentum in deep learning.
\newblock In \emph{International Conference on Machine Learning}, 2013.

\bibitem[Welling \& Teh(2011)Welling and Teh]{SGLD}
Max Welling and Yee~Whye Teh.
\newblock Bayesian learning via stochastic gradient langevin dynamics.
\newblock In \emph{International Conference on Machine Learning}, 2011.

\bibitem[Wu et~al.(2020)Wu, Hu, Xiong, Huan, Braverman, and Zhu]{noisy-GD}
Jingfeng Wu, Wenqing Hu, Haoyi Xiong, Jun Huan, Vladimir Braverman, and Zhanxing Zhu.
\newblock On the noisy gradient descent that generalizes as sgd.
\newblock In \emph{International Conference on Machine Learning}, 2020.

\bibitem[Zhu et~al.(2019)Zhu, Wu, Yu, Wu, and Ma]{anisotropic}
Zhanxing Zhu, Jingfeng Wu, Bing Yu, Lei Wu, and Jinwen Ma.
\newblock The anisotropic noise in stochastic gradient descent: Its behavior of escaping from sharp minima and regularization effects.
\newblock In \emph{International Conference on Machine Learning}, 2019.

\end{thebibliography}
\bibliographystyle{tmlr}

\appendix
\section{Supportive Information}

\subsection{Normalization by Momentum Coefficient}
We verify that the velocity vector is normalized by the choice of the momentum coefficient for Poisson SGD and BPS.

\begin{proposition} \label{lem:V_norm}
    Consider the update \eqref{velocity_udpate} for $v_k$ with its momentum coefficient \eqref{def:alpha_poisson}.
    Then, for $\forall k\in \{1,2,...,K\}$, we have $\|v_k\|  = 1$.
    Further, for $\hat{v}_k$ defined in Algorithm \ref{algorithm: BPS}, we obtain $\|\hat{v}_k\| = 1 $ for every $k = 1,...,K$.
\end{proposition}

\begin{proof}
We first consider $v_k$ with the Poisson SGD case.
Simply, we have
\begin{align}
    \|v_k\| =&\left\|v_{k-1}-2\frac{\langle \nabla \hat{L}_{\textbf{z}}^{(k)}(\theta_k),v_{k-1}\rangle}{\|\nabla \hat{L}_{\textbf{z}}^{(k)}(\theta_k)\|^2}\nabla \hat{L}_{\textbf{z}}^{(k)}(\theta_k)\right\| \\
    =&\left\|\left(I_d-2\frac{\nabla \hat{L}_{\textbf{z}}^{(k)}(\theta_k)\nabla \hat{L}_{\textbf{z}}^{(k)}(\theta_k)^\top}{\|\nabla \hat{L}_{\textbf{z}}^{(k)}(\theta_k)\|^2}\right)v_{k-1}\right\| \\
    =&\sqrt{v_{k-1}^\top\left(I_d-2\frac{\nabla \hat{L}_{\textbf{z}}^{(k)}(\theta_k)\nabla \hat{L}_{\textbf{z}}^{(k)}(\theta_k)^\top}{\|\nabla \hat{L}_{\textbf{z}}^{(k)}(\theta_k)\|^2}\right)^2v_{k-1}} \\
    =&\|v_{k-1}\|.
\end{align}
Since we set $\|v_0\| = 1$ for initialization, the statement holds.

For $\hat{v}_k$ with the BPS case, the reflection does not change the norm of $\hat{v}_k$ in the same way, and the refreshment also keeps $\|\hat{v}_k\|=1$, which completes the proof. 
\end{proof}

\subsection{Details of Experiments in Section \ref{section: experiments}} \label{sec:detail_experiment}

\subsubsection{Details of Section \ref
{sec:exp_stationary}}

As a design of the experiment, we set a distribution of a sample $z = (x,y)$ as $x = (x_1,x_2)\sim (\mathrm{Unif}(-5,5))^{\otimes 2}$ and $y|x\sim\mathcal{N}(x_1^2+x_2^2, 1)$ hold. 
For training, we set $d = 2$ and consider a parameter $\theta = (\theta_1, \theta_2) \in \Theta$ and a loss function $\ell(z;\theta) = ((\theta_1^2-\theta_2^2)x_1^2+(\theta_1^2-\theta_2^2)x_2^2 - y)^2$,
then update the parameters by Poisson SGD with $n=10000$ samples, batch size $m=100$. For the training of SGD, we set its learning rate as $0.002$. 
Since the empirical risk $L_{\textbf{z}}(\theta)$ is approximated by $\mathbb{E}_{\varepsilon\sim\mathcal{N}(0,1),x_1,x_2}[\{(\theta_1^2-\theta_2^2)x_1^2+(\theta_1^2-\theta_2^2)x_2^2-(x_1^2+x_2^2+\varepsilon)\}^2]=(\theta_1^2-\theta_2^2-1)^2\mathbb{E}_{x_1,x_2}[\{(x_1^2+x_2^2)^2]+(const)$ and the dominant term of $\mu_{\textbf{z}}^{(\beta, \varepsilon)}$ is $\exp(-\beta L_{\textbf{z}}(\theta))$, the stationary distribution $\mu_{\textbf{z}}^{(\beta, \varepsilon)}$ is approximated by $\exp(-\beta'(\theta_1^2-\theta_2^2-1)^2)$, where $\beta'$ is a constant multiple of $\beta$. 

\subsubsection{Details of MNIST Experiment in Section \ref{sec:exp_opt}}

We consider a fully connected neural network of $4$ layers, and each layer has $200$ units and the sigmoid activation function. We compare the performance of Poisson SGD with SGD, SGD with Momentum, and SGLD. We set the batch size as $256$, the learning rate of the SGD, the SGD with momentum and SGLD as $0.01$, and the momentum coefficient as $0.9$. We choose the hyperparameter of Poisson SGD as $C_P=100$, $C_{\alpha}=100$, and $\beta=10000$. We also use $\beta=10000$ for SGLD. We use $60000$ images for training and $10000$ images for validation. 

\subsubsection{Details of CIFAR10 Experiment in Section \ref{sec:exp_opt}}

We trained a convolutional neural network of $3$ layers with the ReLU activation function, a $3\times 3$ kernel, and the dropout rate $0.25$. We compare the performance of Poisson SGD with SGD, SGD with Momentum, SGLD. 
We set the batch size as 256, the learning rate of SGD, SGD with momentum, and SGLD as 0.01, and the momentum coefficient as 0.9. We choose the hyperparameter of Poisson SGD as $C_P=100$, $C_{\alpha}=1$, and $\beta=10000$. We also use $\beta=10000$ for SGLD. We use 45000 data points for training and 5000 data points for validation.

\subsubsection{Details of CIFAR100 Experiment for VGG16 in Section \ref{sec:exp_opt}}

We adapted the VGG16 \citep{simonyan2015verydeep} for training. We set the batch size as 128 and the learning rate of SGD as 0.01. We select the hyperparameter of Poisson SGD as $C_P=100$, $C_{\alpha}=0.01$, and $\beta=10000$. We use 50000 data points for training and 10000 data points for validation.

\subsubsection{Details of CIFAR100 Experiment for ResNet18 in Section \ref{sec:exp_opt}}

We adapted the ResNet18 \citep{he2015resnet} for training. We set the batch size as 128 and the learning rate of SGD and Adam as 0.01. As for the hyperparameter for Adam, we set $\beta_1=0.9, \beta_2=0.999$ and $\epsilon=10^{-8}$. We select the hyperparameter of Poisson SGD as $C_P=100$, $C_{\alpha}=0.01$, and $\beta=10000$. We use 50000 data points for training and 10000 data points for validation.

\section{Proof of Theorem \ref{thm:stationary_Poisson_SGD}}
\begin{proof}
By Theorem \ref{SGD-GD} and \ref{thm: stationary-distribution-of-BPS}, we can bound the approximation error
\begin{align}
    &\mathcal{W}_1(\mu_{\textbf{z},K}, \hat{\mu}_{\textbf{z},K})\leq 4\sqrt{d}K\varepsilon,
\end{align}
and the convergence error of BPS as
\begin{align}
    &\|\hat{\mu}_{\textbf{z},K}-\mu_{\textbf{z}}^{(\beta, \varepsilon)}\|_{\mathrm{TV}} \leq \kappa(\beta, \varepsilon,d)^K.
\end{align}
From Theorem 4 in \cite{Wasserstein-TV} (explicit form is Theorem \ref{thm: Wasserstein-TV} in Appendix \ref{section: cite_theorem}), we can bound the Wasserstein distance by the total variation, then obtain
\begin{align}
    \mathcal{W}_1(\hat{\mu}_{\textbf{z},K},\mu_{\textbf{z}}^{(\beta,\varepsilon)})\leq W\|\hat{\mu}_{\textbf{z},K}-\mu_{\textbf{z}}^{(\beta, \varepsilon)}\|_{\mathrm{TV}}\leq W\kappa(\beta,\varepsilon,d)^K.
\end{align}
The triangle inequality completes the proof.
\end{proof}

\section{Proof of Theorem \ref{SGD-GD}}

\begin{proof}
From the definition of Wasserstein distance,
\begin{equation*}
    \mathcal{W}_1(\mu_{\textbf{z},k}, \hat{\mu}_{\textbf{z},k})=\inf_{\pi \in \Pi(\mu_{\textbf{z},k},\hat{\mu}_{\textbf{z},k})} \mathbb{E}_{\pi}[\|\theta_k - \hat{\theta}_k\|_1]
\end{equation*}
holds, so we study the distance between $\theta_k$ and $\hat{\theta}_k$ in terms of the norm $\|\cdot\|_1$. 
Since $\|v_k\|=\|v_{k-1}\|=\|\hat{v}_k\|=\|\hat{v}_{k-1}\|=1$ holds by Proposition \ref{lem:V_norm}, we have
\begin{align}
\label{theta-inequality}
    \mathbb{E}_{\pi}[\|\theta_k-\hat{\theta}_k\|_1] =&\mathbb{E}_{\pi}[\|\theta_{k-1}+\eta_k v_{k-1}-(\hat{\theta}_{k-1}+\hat{\eta}_k \hat{v}_{k-1})\|_1] \notag \\
    \leq&\mathbb{E}_{\pi}[\|\theta_{k-1}-\hat{\theta}_{k-1}\|_1]+\mathbb{E}_{\pi}[\|(\hat{\eta}_k-\eta_k)\hat{v}_{k-1}+\eta_k(\hat{v}_{k-1}-v_{k-1})\|_1] \notag \\
    \leq&\mathbb{E}_{\pi}[\|\theta_{k-1}-\hat{\theta}_{k-1}\|_1]+\mathbb{E}_{\pi}[\|(\hat{\eta}_k-\eta_k)\hat{v}_{k-1}\|_1]+\mathbb{E}_{\pi}[\|\eta_k(\hat{v}_{k-1}-v_{k-1})\|_1] \notag \\
    \leq&\mathbb{E}_{\pi}[\|\theta_{k-1}-\hat{\theta}_{k-1}\|_1]+\sqrt{d}\mathbb{E}_{\pi}[|\eta_k-\hat{\eta}_k|]+2\sqrt{d}\mathbb{E}_{\pi}[\eta_k],
\end{align}
where we use $\|\cdot\|_1\le\sqrt{d}\|\cdot\|$ in the last inequality.

We first evaluate the second term of \eqref{theta-inequality}.
There exists a coupling $\pi$ such that
\begin{equation*}
    \mathbb{E}_{\pi}[|\eta_k-\hat{\eta}_k|]=\mathcal{W}_1(P_{\eta_k},P_{\hat{\eta}_k})
\end{equation*}
holds, where $P_{\eta_k}$ and $P_{\hat{\eta}_k}$ denote the distribution of $\eta_k$ and $\hat{\eta}_k$ respectively.  We use such a coupling as $\pi$.
In evaluating $\mathcal{W}_1(P_{\eta_k},P_{\hat{\eta}_k})$, we consider the following analysis.
$\eta_k$ and $\hat{\eta}_k$ are 1-dimensional and their cumulative distribution function is written as
\begin{align}
    &F_1(t)=1-\exp\left(-\int_0^t (\beta\langle \nabla \hat{L}_{\textbf{z}}^{(k)}(\theta+rv),v\rangle_++C_P)dr\right), \\
    &F_2(t)=1-\exp\left(-\int_0^t (\beta\langle \nabla L_{\textbf{z}}(\theta+rv),v\rangle_++C_B+\Lambda_{\mathrm{ref}})dr\right),
\end{align}
respectively, and we also have 
\begin{align}
    &\beta\langle \nabla \hat{L}_{\textbf{z}}^{(k)}(\theta+rv),v\rangle_++C_P\ge C_P, \\
    &\beta\langle \nabla L_{\textbf{z}}(\theta+rv),v\rangle_++C_B+\Lambda_{\mathrm{ref}}\ge C_B+\Lambda_{\mathrm{ref}}, \mbox{~and~}\\
    &|(\beta\langle \nabla \hat{L}_{\textbf{z}}^{(k)}(\theta+rv),v\rangle_++C_P)-(\beta\langle \nabla L_{\textbf{z}}(\theta+rv),v\rangle_++C_B+\Lambda_{\mathrm{ref}})| \\
    &\le \max\{|-\beta M_{\ell}+C_P-(C_B+\Lambda_{\mathrm{ref}})|,|\beta M_{\ell}+C_P-(C_B+\Lambda_{\mathrm{ref}})|\}.
\end{align}
Hence, we can use Lemma \ref{lem: 1-dim Wasserstein} and obtain
\begin{align}
\label{result-1}
    \mathcal{W}_1(P_{\eta_k},P_{\hat{\eta}_k})\le\frac{\max\{|-\beta M_{\ell}+C_P-(C_B+\Lambda_{\mathrm{ref}})|,|\beta M_{\ell}+C_P-(C_B+\Lambda_{\mathrm{ref}})|\}}{C_P(C_B+\Lambda_{\mathrm{ref}})}.
\end{align}

Next, we evaluate the third term of \eqref{theta-inequality}. 
We have
\begin{align}\label{result-2}
\mathbb{E}[\eta_k] = &\int_0^{\infty} P(\eta_k\geq t)dt \notag\\
=&\int_0^{\infty} \exp\left(-\int_{0}^{t}\{\beta\langle \nabla \hat{L}^{(k)}_{\textbf{z}}(\theta_{k-1}+rv_{k-1}), v_{k-1}\rangle_++C_P\}dr\right)dt \notag\\
\leq&\int_0^{\infty} \exp\left(-C_Pt\right)dt \notag\\
=&\frac{1}{C_P}.
\end{align}

Substituting \eqref{result-1} and \eqref{result-2} into \eqref{theta-inequality}, we have
\begin{align}
\label{theta-result}
    \mathbb{E}_{\pi}[\|\theta_k-\hat{\theta}_k\|_1]
    \leq&\mathbb{E}_{\pi}[\|\theta_{k-1}-\hat{\theta}_{k-1}\|_1] \\ 
    +&\frac{\sqrt{d}\max\{|-\beta M_{\ell}+C_P-(C_B+\Lambda_{\mathrm{ref}})|,|\beta M_{\ell}+C_P-(C_B+\Lambda_{\mathrm{ref}})|\}}{C_P(C_B+\Lambda_{\mathrm{ref}})}+\frac{2\sqrt{d}}{C_P}.
\end{align}

Since we take $C_P$ in Poisson SGD as $C_P={1} / {\varepsilon}$
and $C_B$ and $\Lambda_{\mathrm{ref}}$ in BPS as $C_B+\Lambda_{\mathrm{ref}}=\beta M_{\ell}+{1} / {\varepsilon}$,
\eqref{theta-result} can be written as
\begin{align}
    \mathbb{E}_{\pi}[\|\theta_k-\hat{\theta}_k\|_1]
    \leq&\mathbb{E}_{\pi}[\|\theta_{k-1}-\hat{\theta}_{k-1}\|_1]+4\sqrt{d}\varepsilon.
\end{align}
Hence, solving this recursive inequality with $\theta_0=\hat{\theta}_0$, we have
\begin{align}
    \mathbb{E}_{\pi}[\|\theta_K-\hat{\theta}_K\|_1]
    \leq4\sqrt{d}K\varepsilon,
\end{align}
which is the desired conclusion.
\end{proof}

\begin{lemma}\label{lem: 1-dim Wasserstein}
Let $a_1$ and $a_2$ be $\mathbb{R}$-valued random variables whose cumulative distribution functions are
\begin{align}
    &F_1(t)=1-\exp\left(-\int_{0}^t f_1(r)dr\right), \mbox{~and~}F_2(t)=1-\exp\left(-\int_{0}^t f_2(r)dr\right),
\end{align}
respectively, where $f_1,f_2:\mathbb{R}\rightarrow\mathbb{R}$ are continuous functions. Let the distributions of $a_1$ and $a_2$ be $P_1$ and $P_2$ respectively. Suppose that there exists $M,m_1,m_2>0$ such that $|f_2(t)-f_1(t)|\le M$, $m_1\le f_1(t)$, and $m_2\le f_2(t)$ hold for $\forall t\in\mathbb{R}$. Then, the Wasserstein distance between $P_1$ and $P_2$ satisfies
\begin{align}
    \mathcal{W}_1(P_1,P_2)\le \frac{M}{m_1 m_2}.
\end{align}
\end{lemma}

\begin{proof}
Since $a_1$ and $a_2$ are 1-dimensional, we have
\begin{align}
    \mathcal{W}_1(P_1,P_2)  =\int_0^1 \left|F_1^{-1}(q)-F_2^{-1}(q)\right|dq.
\end{align}
We introduce several notation $\delta(r)=f_2(r)-f_1(r)$, $t=F_1^{-1}(q)$, and $t'=F_2^{-1}(q)$, then
\begin{align}
    &\int_{0}^t f_1(r)dr=\log \frac{1}{1-q} \\
    &\int_{0}^{t'} (f_1(r)+\delta(r))dr=\log \frac{1}{1-q}
\end{align}
holds.
So, we obtain
\begin{align}
    \int_{t'}^{t} f_1(r)dr=\int_0^{t'}\delta(r)dr.
\end{align}
Hence, we have
\begin{align}
    \left|\int_{t'}^{t} f_1(r)dr\right|=\int_{\min\{t,t'\}}^{\max\{t,t'\}}f_1(r)dr\leq Mt'.
\end{align}
In addition, $\int_{\min\{t,t'\}}^{\max\{t,t'\}}f_1(r)dr \geq m_1|t-t'|$ holds, so we have
\begin{align}
    |t-t'|\leq \frac{Mt'}{m_1}.
\end{align}
We have the upper bound of $t'$ as
\begin{align}
    \log \frac{1}{1-q}=\int_{0}^{t'} f_2(r)dr
    \geq m_2 t',
\end{align}
so we have
\begin{align}
    |t-t'|\leq \frac{M}{m_1 m_2}\log\frac{1}{1-q}.
\end{align}
Since $\int_0^1 |\log(1-q)|dq = 1$ holds, we obtain
\begin{align}
    \int_0^1 \left|F_1^{-1}(q)-F_2^{-1}(q)\right|dq\leq \frac{M}{m_1 m_2}.
\end{align}
\end{proof}

\section{Proof of Theorem \ref{thm: stationary-distribution-of-BPS}}
We prove this theorem by two steps. First, we prove that BPS has $\mu_{\textbf{z}}^{(\beta,\varepsilon)}$ as one of its stationary distributions in section \ref{subsection: 1}. At this stage, BPS may have other forms of stationary distribution or may not converge to its stationary distribution. Second, we prove that BPS has a unique stationary distribution and converges to its stationary distribution at exponential rate, in other words, it has the exponential ergodicity, in section \ref{subsection: 2}.
\subsection{The form of the stationary distribution}\label{subsection: 1}
In this section, we check that BPS has $\mu_{\textbf{z}}^{(\beta,\varepsilon)}$ as a stationary distribution.
In the proof, we define $\lambda(\theta,v):=\beta\langle \nabla L_{\textbf{z}}(\theta),v\rangle_+$, $\bar{\lambda}(\theta,v):=\lambda(\theta,v)+\Lambda_\mathrm{ref}$, and $R_{\textbf{z}}(\theta):=I_d-2\frac{\nabla L_{\textbf{z}}(\theta)\nabla L_{\textbf{z}}(\theta)^\top}{\|\nabla L_{\textbf{z}}(\theta)\|^2}$. We remark that $R_{\textbf{z}}$ is a symmetric matrix and satisfies $R_{\textbf{z}}(\theta)^2=I_d$, so it is also an orthogonal matrix.

From the proof of Lemma 1 in the supplementary material of \cite{Deligiannidis2017}, we can write the transition probability $\hat{Q}$ of BPS as following for arbitrary measurable sets $A\subset \Theta$ and $B\subset \mathbb{S}^{d-1}$:
\begin{align}
\hat{Q}((\theta,v), A\times B)=&\int_{0}^{\infty}\exp\left\{-\int_{0}^{s}\left(\bar{\lambda}(\theta+uv,v)+C_B\right)du\right\} \\
&\times \left(\bar{\lambda}(\theta+uv,v)+C_B\right)K((\theta+sv,v),A\times B)ds, \label{def:Qhat}
\end{align}
where a transition kernel $K$ is expressed as
\begin{align}\label{equation:K}
    K((\theta,v),A\times B)=&\frac{\lambda(\theta,v)+C_B}{\bar{\lambda}(\theta,v)+C_B}\mathbbm{1}[\theta\in A]\mathbbm{1}[R_{\textbf{z}}(\theta)v\in B] \\ +&\frac{\Lambda_{\mathrm{ref}}}{\bar{\lambda}(\theta,v)+C_B}\mathbbm{1}[\theta\in A]\mu_\mathrm{unif}(B),
\end{align}
where $\mu_\mathrm{unif}$ is the uniform probability measure on $\mathbb{S}^{d-1}$.

\begin{lemma}
\label{lem: stationary-of-BPS}
Under Assumption \ref{assumption: loss-function}, a probability measure on $\Theta\times \mathbb{S}^{d-1}$
\begin{equation*}
\hat{\mu}_{\textbf{z}}(A\times B)\propto\int_{A\times B}\left(\bar{\lambda}(\theta,-v)+C_B\right) \exp(-\beta L_{\textbf{z}}(\theta))d\theta\mu_\mathrm{unif}(dv)
\end{equation*}
is the stationary distribution induced from the transition probability $\hat{Q}$ as \eqref{def:Qhat}.
\end{lemma}

\begin{proof}
Our proof is almost the same as the proof of Lemma 1 in \cite{Deligiannidis2017}.
Let $\pi_{\textbf{z}}(d\theta,dv)=\exp(-\beta L_{\textbf{z}}(\theta))d\theta\mu_\mathrm{unif}(dv)$. 

First, we prove
\begin{equation}
\label{2}
\int (\bar{\lambda}(\theta,v)+C_B)\pi_{\textbf{z}}(d\theta,dv)K((\theta,v),A\times B)\propto\hat{\mu}_{\textbf{z}}(A\times B).
\end{equation}
Substituting \eqref{equation:K}, the left side of \eqref{2} is rewritten as
\begin{align}
    \int \pi_{\textbf{z}}(d\theta,dv)(\lambda(\theta,v)+C_B)\mathbbm{1}[\theta\in A]\mathbbm{1}[R_{\textbf{z}}(\theta)v\in B]+\int \pi_{\textbf{z}}(d\theta,dv)\Lambda_{\mathrm{ref}}\mathbbm{1}[\theta\in A]\mu_\mathrm{unif}(B).
\end{align}
We consider changing the variable as $v'=R_{\textbf{z}}(\theta)v$. 
Since $R_{\textbf{z}}(\theta)^{-1}=R_{\textbf{z}}(\theta)$ holds, we get $\lambda(\theta,R_{\textbf{z}}(\theta)^{-1}v')=\lambda(\theta,-v')$. 
In addition, since $|\mathrm{det}(R_{\textbf{z}}(\theta))|=1$, and $\mu_\mathrm{unif}(R_{\textbf{z}}(\theta)^{-1}dv')=\mu_\mathrm{unif}(dv')$ hold due to the rotational invariance of $\mu_\mathrm{unif}$, we obtain
\begin{equation*}
\int_{A\times B} \pi_{\textbf{z}}(d\theta,dv')(\lambda(\theta,-v')+C_B)+\int_{A\times B} \pi_{\textbf{z}}(d\theta,dv')\Lambda_{\mathrm{ref}},
\end{equation*}
which is proportional to the right side of \eqref{2} from the definition of $\hat{\mu}_{\textbf{z}}$.

Second, we prove $\int \hat{Q}((\theta,v),(dy,dw))\hat{\mu}_{\textbf{z}}(d\theta,dv)=\hat{\mu}_{\textbf{z}}(dy,dw)$.
We have
\begin{align}
&\int \hat{Q}((\theta,v),(dy,dw))\hat{\mu}_{\textbf{z}}(d\theta,dv) \\
\propto&\int_{0}^{\infty} \exp\left(-\int_{0}^{s}\{\bar{\lambda}(\theta+uv,v)+C_B\}du\right)\{\bar{\lambda}(\theta+sv,v)+C_B\} \\
&\times K((\theta+sv,v),(dy,dw))\{\bar{\lambda}(\theta,-v)+C_B\}\pi_{\textbf{z}}(d\theta,dv)ds. \\
\end{align}
If we change $\theta$ as $t=\theta+sv$, then this integral becomes
\begin{align}
&\int_{0}^{\infty}\exp\left(-\int_{0}^{s}\{\bar{\lambda}(t+(u-s)v,v)+C_B\}du\right)\{\bar{\lambda}(t,v)+C_B\} \\
&\quad \times K((t,v),(dy,dw))\{\bar{\lambda}(t-sv,-v)+C_B\}\pi_{\textbf{z}}(d\theta,dv)ds. \\
\end{align}
Since $L_{\textbf{z}}(\theta)$ is absolutely continuous, 
\begin{equation*}
\exp(-\beta L_{\textbf{z}}(t-sv))=\exp\left(-\beta L_{\textbf{z}}(t)-\int_{0}^{s}\lambda(t-wv,-v)dw+\int_{0}^{s}\lambda(t-wv,v)dw\right)
\end{equation*}
holds in the same way as \cite{Deligiannidis2017}. Substituting it into $\pi_{\textbf{z}}(dx,dv)$ and changing $u$ as $u-s=-w$,
\begin{align}
&\int_{0}^{\infty} \exp\left(-\int_{0}^{s}\{\bar{\lambda}(t-wv,-v)+C_B\}dw\right)\{\bar{\lambda}(t-sv,-v)+C_B\}ds \\
\times &\{\bar{\lambda}(t,v)+C_B\}K((t,v),(dy,dw))\pi_{\textbf{z}}(dt,dv)
\end{align}
holds. The first line can be calculated as $\left[-\exp\left(-\int_{0}^{s}\{\bar{\lambda}(t-wv,-v)+C_B\}dw\right)\right]_0^{\infty}=1$, so it is equal to
\begin{equation*}
\int \{\bar{\lambda}(t,v)+C_B\}K((t,v),(dy,dw))\pi_{\textbf{z}}(dt,dv).
\end{equation*}
Using \eqref{2}, it is proportional to $\hat{\mu}_{\textbf{z}}(dy,dw)$, which completes the proof.
\end{proof}

By the following proposition, we prove that $\mu_{\textbf{z}}^{(\beta,\varepsilon)}$ is one of the stationary distributions of BPS.
Recall that we defined $a_d:={\Gamma(d/2)} / ({\sqrt{\pi}\Gamma(d/2+1/2)})$.
\begin{proposition}
\label{marginal}
    The marginal distribution of the stationary distribution expressed in Lemma \ref{lem: stationary-of-BPS} is written as
    \begin{equation*}
        \hat{\mu}_{\textbf{z}}(d\theta)\propto (\Lambda_{\mathrm{ref}}+C_B+a_d\beta\|\nabla L_{\textbf{z}}(\theta)\|)\exp(-\beta L_{\textbf{z}}(\theta))d\theta.
    \end{equation*}
Hence, if we put $\Lambda_{\mathrm{ref}}$ and $C_B$ as $\Lambda_\mathrm{ref}+C_B=\beta M_{\ell}+{1} / {\varepsilon}$, it corresponds to $\mu_{\textbf{z}}^{(\beta,\varepsilon)}$.
\end{proposition}

\begin{proof}
We only need to integrate with $v$ the distribution $\hat{\mu}_{\textbf{z}}$ expressed in Lemma \ref{lem: stationary-of-BPS}. We have
\begin{align}
    \hat{\mu}_{\textbf{z}}(d\theta)\propto&\int_{v\in \mathbb{S}^{d-1}}(\Lambda_{\mathrm{ref}}+C_B+\beta\langle \nabla L_{\textbf{z}}(\theta),-v\rangle_+)\exp(-\beta L_{\textbf{z}}(\theta))d\theta\mu_\mathrm{unif}(dv) \\
    =&(\Lambda_{\mathrm{ref}}+C_B)\exp(-\beta L_{\textbf{z}}(\theta))d\theta+\exp(-\beta L_{\textbf{z}}(\theta))d\theta\beta \mathbb{E}_{v\sim\mu_\mathrm{unif}}[\langle \nabla L_{\textbf{z}}(\theta),v\rangle_+].
\end{align}
We can calculate the expected value in the last term as
\begin{align}
    &\mathbb{E}_{v\sim\mu_\mathrm{unif}}[\langle \nabla L_{\textbf{z}}(\theta),v\rangle_+] =\mathbb{E}_{v\sim\mu_\mathrm{unif}}[\|\nabla L_{\textbf{z}}(\theta)\|(\cos{\phi})_+] =\|\nabla L_{\textbf{z}}(\theta)\|\mathbb{E}_{v\sim\mu_\mathrm{unif}}[(\cos{\phi})_+],
\end{align}
where $\phi \in \mathbb{R}$ is a random variable dependent on $v$ which satisfies
\begin{equation}
\label{cosine}
    \cos{\phi}=\left\langle \frac{\nabla L_{\textbf{z}}(\theta)}{\|\nabla L_{\textbf{z}}(\theta)\|},v\right\rangle.
\end{equation}
From the symmetry of the uniform distribution, we can calculate $\mathbb{E}_{v\sim\mu_\mathrm{unif}}[(\cos{\phi})_+]$ by replacing $\frac{\nabla L_{\textbf{z}}(\theta)}{\|\nabla L_{\textbf{z}}(\theta)\|}$ in \eqref{cosine} by $(1,0,\cdots,0)$. Hence,
\begin{align}
    \mathbb{E}_{v\sim\mu_\mathrm{unif}}[(\cos{\phi})_+]=\mathbb{E}_{v\sim\mu_\mathrm{unif}}[(v_1)_+]=\mathbb{E}\left[\left(\frac{x_1}{\sqrt{x_1^2+\cdots+x_d^2}}\right)_+\right]
\end{align}
holds, where $v_1$ is the first component of $v$ and $x_i (i=1,...,d)$ is \textit{i.i.d.} standard Gaussian variables.

For $(x_{1},\ldots,x_{d})\sim \mathcal{N}(\mathbf{0},I_{d})$, we have
\begin{align}
    \Ep\left[\sqrt{\frac{x_{1}^{2}}{x_{1}^{2}+\cdots+x_{d}^{2}}}\right]
    &=\int_{\R^{d}}\sqrt{\frac{z_{1}^{2}}{z_{1}^{2}+\cdots+z_{d}^{2}}}\frac{1}{(2\pi)^{d/2}}\exp\left(-\frac{z_{1}^{2}+\cdots+z_{d}^{2}}{2}\right)dz_{1}\cdots dz_{d}\\
    &=\int_{[0,\infty)^{2}}\sqrt{\frac{r}{r+s}}\frac{r^{-1/2}\exp\left(-r/2\right)}{\sqrt{2\pi}}\frac{s^{(d-1)/2-1}\exp\left(-s/2\right)}{\Gamma((d-1)/2)2^{(d-1)/2}}drds\\
    &=\int_{[0,1]}t^{1/2}\frac{t^{1/2-1}(1-t)^{(d-1)/2-1}}{\mathrm{B}(1/2,(d-1)/2)}d t\\
    &=\frac{\mathrm{B}(1,(d-1)/2)}{\mathrm{B}(1/2,(d-1)/2)}\\
    &=\frac{\Gamma(1)\Gamma((d-1)/2)\Gamma(d/2)}{\Gamma(1/2)\Gamma((d-1)/2)\Gamma(d/2+1/2)}\\
    &=\frac{\Gamma(d/2)}{\sqrt{\pi}\Gamma(d/2+1/2)}.
\end{align}
Note that for all $d\ge2$,
\begin{equation}
    \frac{1}{\sqrt{d/2}}\le \frac{\Gamma(d/2)}{\Gamma(d/2+1/2)}\le \frac{1}{\sqrt{d/2-1/2}}
\end{equation}
holds (e.g., see \cite{qi2013bounds}).
Therefore, for all $d\ge 2$, we have
\begin{equation}
    \Ep\left[\left(\frac{x_{1}}{\sqrt{x_{1}^{2}+\cdots+x_{d}^{2}}}\right)_{+}\right]=\frac{\Gamma(d/2)}{2\sqrt{\pi}\Gamma(d/2+1/2)}\in\left[\frac{1}{\sqrt{2\pi d}},\frac{1}{\sqrt{2\pi(d-1)}}\right].
\end{equation}

\end{proof}

\subsection{The exponential ergodicity of BPS}\label{subsection: 2}

The next proposition is on the minorization condition of the 2-skeletons of BPS on the restricted domains. In short, minorization means that the stochastic process can go from any measurable set to any measurable set in the parameter space, which is a sufficient condition for the exponential ergodicity in the compact parameter space. 2-Skeleton means 2 step of the stochastic process.
This proposition completes the proof of Theorem \ref{thm: stationary-distribution-of-BPS}. 
\begin{proposition}\label{proposition:minorization}
    Under Assumption \ref{assumption: loss-function}, the $2$-skeletons of BPS satisfies the minorization condition; that is, for some $c>0$, for all $(\theta,v)\in\Theta\times \mathbb{S}^{d-1}$ and all measurable $E\subset\Theta\times\mathbb{S}^{d-1}$, we have
    \begin{align}
        \hat{Q}^{2}((\theta,v),E)&\ge c\int_{\Theta}\int_{\mathbb{S}^{d-1}}\mathbbm{1}[(\theta,v)\in E]d\theta\mu_{\mathrm{unif}}(dv).
    \end{align}
    Moreover, BPS is exponentially ergodic in total variation distance.
\end{proposition}

\begin{proof}
    We partially follow the proof of Lemma 4 in \cite{Deligiannidis2017}.

    Let $f:\Theta\times \mathbb{S}^{d-1}\to[0,\infty)$ be a non-negative and bounded function.
    We also use the notation $M'=\sup_{(\theta,v)\in\Theta\times \mathbb{S}^{d-1}}(\bar{\lambda}(\theta,v)+C_B)<\infty$.
    By considering the event where the first update of $v$ is \textit{refreshment} from $\mathrm{Unif}(\mathbb{S}^{d-1})$, we see that for any $(\theta_{0},v_{0})\in\Theta\times\mathbb{S}^{d-1}$,
    \begin{align}
        &\int_{\Theta\times \mathbb{S}^{d-1}}f(\theta,v)\hat{Q}^{2}((\theta_{0},v_{0}),(d\theta,dv))\\
        &=\int_{\Theta\times \mathbb{S}^{d-1}}\int_{\Theta\times \mathbb{S}^{d-1}}f(\theta,v)\hat{Q}((\theta_{1},v_{1}),(d\theta ,dv))\hat{Q}((\theta_{0},v_{0}),(d\theta_{1},dv_{1}))\\
        &\ge \frac{\Lambda_{\mathrm{ref}}}{M'}\inf_{\theta_{1}\in\Theta}\int_{\Theta\times \mathbb{S}^{d-1}}f(\theta,v)\hat{Q}((\theta_{1},v_{1}),(d\theta,dv))\mu_{\mathrm{unif}}(dv_{1})
    \end{align}
    holds.
    We also obtain that for $T\sim\mathrm{Exp}(M')$, $V_{1},V_{2}\sim^{\mathrm{i.i.d.}}\mathrm{Unif}(\mathbb{S}^{d-1})$, we have
    \begin{align}
        &\inf_{\theta_{1}\in\Theta}\int_{\Theta\times \mathbb{S}^{d-1}}f(\theta,v)\hat{Q}((\theta_{1},v_{1}),d\theta dv)\mu_{\mathrm{unif}}(dv_{1})\\
        &\ge \inf_{\theta_{1}\in\Theta}\frac{\Lambda_{\mathrm{ref}}}{M'}\Ep\left[\mathbbm{1}[\theta_{1}+TV_{1}\in\Theta]f(\theta_{1}+TV_{1},V_{2})\right]\\
        &\ge \inf_{\theta_{1}\in\Theta}\frac{\Lambda_{\mathrm{ref}}^{2}}{M'}\int_{[0,\infty)\times\mathbb{S}^{d-1}\times \mathbb{S}^{d-1}}\mathbbm{1}[\theta_{1}+t v_{1}\in\Theta]e^{-M't}f(\theta_{1}+t v_{1},v)dt\mu_{\mathrm{unif}}(dv_{1})\mu_{\mathrm{unif}}(dv)\\
        &\ge \inf_{\theta_{1}\in\Theta}\frac{\Lambda_{\mathrm{ref}}^{2} e^{-M'\mathrm{diam}(\Theta)}}{M'}\int_{[0,\infty)\times\mathbb{S}^{d-1}\times \mathbb{S}^{d-1}}\mathbbm{1}[\theta_{1}+t v_{1}\in\Theta]f(\theta_{1}+t v_{1},v)dt\mu_{\mathrm{unif}}(dv_{1})\mu_{\mathrm{unif}}(dv)\\
        &=  \inf_{\theta_{1}\in\Theta}\frac{\Lambda_{\mathrm{ref}}^{2} e^{-M'\mathrm{diam}(\Theta)}}{M'}\int_{\Theta\times\mathbb{S}^{d-1}}\mathbbm{1}[\theta\in\Theta]f(\theta,v)\|\theta-\theta_{1}\|^{1-d}d\theta\mu_{\mathrm{unif}}(dv)\\
        &\ge  \frac{\Lambda_{\mathrm{ref}}^{2} e^{-M'\mathrm{diam}(\Theta)}}{M'\mathrm{diam}(\Theta)^{d-1}}\int_{\Theta\times\mathbb{S}^{d-1}}f(\theta,v)d\theta\mu_{\mathrm{unif}}(dv),
    \end{align}
    where the second last equality uses a change of coordinates.
    Since $f$ is generic, the minorization condition holds.
    Harris's theorem thus gives the exponential ergodicity of BPS.
\end{proof}

\section{Proof of Theorem \ref{thm: generalization_error}}
\begin{proof}
We prove in the same way as the proof of Theorem 2.1 in \cite{Raginsky-2017}. 
Let $\theta_{\mu}$ be a random variable satisfying $\theta_{\mu}\sim \mu_{\textbf{z}}^{(\beta,\varepsilon)}$, where $\mu_{\textbf{z}}^{(\beta,\varepsilon)}$ is defined in \eqref{stationary_marginal}.
We denote $\theta_K \sim \mu_{\textbf{z},K}$ as the output of Poisson SGD (Algorithm \ref{algorithm: Poisson-SGD}). We have

\begin{align}
&\mathbb{E}_{\textbf{z}}[\mathbb{E}_{\theta_K}[L(\theta_K)]]-\inf_{\theta\in \Theta} L(\theta) \\
&=\mathbb{E}_{\textbf{z}}[\mathbb{E}_{\theta_K}[L(\theta_K)]-\mathbb{E}_{\theta_{\mu}}[L(\theta_{\mu})]] 
+\{\mathbb{E}_{\textbf{z}}[\mathbb{E}_{\theta_{\mu}}[L(\theta_{\mu})]]-\inf_{\theta\in \Theta} L(\theta)\},
\end{align}
and the second term of right-hand side is written as
\begin{align}
&\mathbb{E}_{\textbf{z}}[\mathbb{E}_{\theta_{\mu}}[L(\theta_{\mu})]]-\inf_{\theta\in \Theta} L(\theta) \\
&=\mathbb{E}_{\textbf{z}}[\mathbb{E}_{\theta_{\mu}}[L(\theta_{\mu})]]-\mathbb{E}_{\textbf{z}}[\mathbb{E}_{\theta_{\mu}}[L_{\textbf{z}}(\theta_{\mu})]]+\left(\mathbb{E}_{\textbf{z}}[\mathbb{E}_{\theta_{\mu}}[L_{\textbf{z}}(\theta_{\mu})]]-\inf_{\theta\in \Theta}L(\theta)\right).
\end{align}
Letting $\theta^\circ=\argmin_{\theta\in\Theta} L(\theta)$, the second part of the right-hand side in the equation above is
\begin{align}
\mathbb{E}_{\textbf{z}}[\mathbb{E}_{\theta_{\mu}}[L_{\textbf{z}}(\theta_{\mu})]]-\inf_{\theta\in \Theta} L(\theta) =&\mathbb{E}_{\textbf{z}}[\mathbb{E}_{\theta_{\mu}}[L_{\textbf{z}}(\theta_{\mu})]-\inf_{\theta\in \Theta}L_{\textbf{z}}(\theta)]+\left(\mathbb{E}_{\textbf{z}}\left[\inf_{\theta\in \Theta} L_{\textbf{z}}(\theta)-L_{\textbf{z}}(\theta^\circ)\right]\right) \\
\leq&\mathbb{E}_{\textbf{z}}[\mathbb{E}_{\theta_{\mu}}[L_{\textbf{z}}(\theta_{\mu})]-\inf_{\theta\in \Theta}L_{\textbf{z}}(\theta)].
\end{align}
As a result, we have
\begin{align}
\mathbb{E}_{\textbf{z}}[\mathbb{E}_{\theta_K}[L(\theta_K)]]-\inf_{\theta\in \Theta} L(\theta) \label{inequality0}
\leq&\mathbb{E}_{\textbf{z}}[\mathbb{E}_{\theta_K}[L(\theta_K)]-\mathbb{E}_{\theta_{\mu}}[L(\theta_{\mu})]] \\
\label{inequality1}
+&\mathbb{E}_{\textbf{z}}[\mathbb{E}_{\theta_{\mu}}[L(\theta_{\mu})]-\mathbb{E}_{\theta_{\mu}}[L_{\textbf{z}}(\theta_{\mu})]] \\
\label{inequality3}
+&\mathbb{E}_{\textbf{z}}[\mathbb{E}_{\theta_{\mu}}[L_{\textbf{z}}(\theta_{\mu})]-\inf_{\theta\in \Theta}L_{\textbf{z}}(\theta)].
\end{align}
To evaluate the terms \eqref{inequality0}, \eqref{inequality1}, and \eqref{inequality3}, we prepare the following lemma to calculate the upper bound of the difference between two expected value by the Wasserstein distance.
\begin{lemma} \label{lem:error_bound}
Consider probability measures $\mu$ and $\nu$ on $\Theta$.
Suppose that $\sup_{z \in \mZ}|\ell(z;0)|\leq A$ and $\sup_{z \in \mZ}\|\nabla \ell(z;0)\|\leq B$ hold.
Then, we obtain
\begin{align}
    &\left|\mathbb{E}_{\theta_1\sim \mu}[\ell(z;\theta_1)]-\mathbb{E}_{\theta_2\sim \nu}[\ell(z;\theta_2)]\right| \leq (c_1W+B)\sqrt{W\mathcal{W}_1(\mu,\nu)},  \mbox{~and~}\label{ell-wasserstein}\\
    &\left|\mathbb{E}_{\theta_1\sim \mu}[L(\theta_1)]-\mathbb{E}_{\theta_2\sim \nu}[L(\theta_2)]\right| \leq (c_1W+B)\sqrt{W\mathcal{W}_1(\mu,\nu)}. \label{L-wasserstein}
\end{align}    
\end{lemma}
\begin{proof}
    
Under the assumption, Lemma 3.1 in \cite{Raginsky-2017} holds. Hence, we have
\begin{align}
\label{upper of gradient}
&\|\nabla \ell(z;\theta)\| \leq c_1\|\theta\|+B, \forall \theta\in\Theta, \forall z\in \mathcal{Z} \\
\label{upper of loss}
&\ell(z;\theta)\leq \frac{c_1}{2}\|\theta\|^2+B\|\theta\|+A, \forall \theta\in\Theta, \forall z\in \mathcal{Z}.
\end{align}
Moreover, from Lemma 3.5 in \cite{Raginsky-2017}, for arbitrary two probability measures $\mu$ and $\nu$, if we let
\begin{equation*}
\sigma^2=\max\{\mathbb{E}_{\theta_1\sim \mu}[\|\theta_1\|^2],\mathbb{E}_{\theta_2\sim \nu}[\|\theta_2\|^2]\},
\end{equation*}
then we have
\begin{equation*}
\left|\mathbb{E}_{\theta_1\sim \mu}[\ell(z;\theta_1)]-\mathbb{E}_{\theta_2\sim \nu}[\ell(z;\theta_2)]\right| \leq (c_1\sigma+B)\mathcal{W}_2(\mu,\nu).
\end{equation*}
Obviously, it also holds that
\begin{equation*}
\left|\mathbb{E}_{\theta_1\sim \mu}[L(\theta_1)]-\mathbb{E}_{\theta_2\sim \nu}[L(\theta_2)]\right| \leq (c_1\sigma+B)\mathcal{W}_2(\mu,\nu).
\end{equation*}
Since we have $\sigma\leq W$ and $\mW_2 (\mu,\nu) =  \inf_{\pi \in \Pi(\mu,\nu)} (\int_{\Theta} \|z - z'\|^2 d \pi(z,z'))^{1/2}\leq \inf_{\pi \in \Pi(\mu,\nu)} (\int_{\Theta} W\|z - z'\|_1 d \pi(z,z'))^{1/2}=\sqrt{W\mathcal{W}_1(\mu,\nu)}$, we obtain the statement.
\end{proof}

We start evaluating each of the terms \eqref{inequality0}, \eqref{inequality1}, and \eqref{inequality3}.

First, we study \eqref{inequality0}. From \eqref{L-wasserstein} in Lemma \ref{lem:error_bound}, we have
\begin{align}
    \mathbb{E}_{\theta_K}[L(\theta_K)]-\mathbb{E}_{\theta_{\mu}}[L(\theta_{\mu})]\leq&(c_1W+B)\sqrt{W\mathcal{W}_1(\mu_{\textbf{z},K},\mu_{\textbf{z}}^{(\beta,\varepsilon)})} \\
    \leq&(c_1W+B)\sqrt{Wd_K(\beta,\varepsilon,d)}. \label{ineq:1_bound}
\end{align}

Second, we evaluate \eqref{inequality1} using the same approach as \cite{Raginsky-2017}. 
Here, we need to evaluate
\begin{align*}
\mathbb{E}_{\theta_{\mu}}[\ell(z;\theta_{\mu})]-\mathbb{E}_{\theta_{\mu'}}[\ell(z;\theta_{\mu'})],
\end{align*}
where $z\in\mathcal{Z}$ is an arbitrary sampled data, $\theta_{\mu'}\sim\mu_{\textbf{z}'}^{(\beta,\varepsilon)}$ and $\mu_{\textbf{z}'}^{(\beta,\varepsilon)}$ is the stationary distribution of BPS when one of the data $z_i$ is changed to arbitrary $\bar{z}_i\in\mathcal{Z}$ and $\textbf{z}'$ is a dataset with replacing $z_i$ to $\bar{z}_i$, and $L_{\textbf{z}'}$ be its corresponding empirical risk.
From \eqref{ell-wasserstein} in Lemma \ref{lem:error_bound}, we have
\begin{align*}
\mathbb{E}_{\theta_{\mu}}[\ell(z;\theta_{\mu})]-\mathbb{E}_{\theta_{\mu'}}[\ell(z;\theta_{\mu'})]\leq &(c_1W+B)\mathcal{W}_2(\mu_{\textbf{z}}^{(\beta,\varepsilon)},\mu_{\textbf{z}'}^{(\beta,\varepsilon)}) \\
\leq &(c_1W+B)C_{\mu'}\left[\sqrt{D(\mu_{\textbf{z}}^{(\beta,\varepsilon)}||\mu_{\textbf{z}'}^{(\beta,\varepsilon)})}+\left(\frac{D(\mu_{\textbf{z}}^{(\beta,\varepsilon)}||\mu_{\textbf{z}'}^{(\beta,\varepsilon)})}{2}\right)^{\frac{1}{4}}\right],
\end{align*}
where $D(\cdot || \cdot )$ is KL-divergence and
\begin{equation*}
C_{\mu'}:=2\inf_{\lambda>0}\left(\frac{1}{\lambda}\left(\frac{3}{2}+\log \int_{\Theta}e^{\lambda\|\theta\|^2}\mu_{\textbf{z}'}^{(\beta,\varepsilon)}(d\theta)\right)\right)^{\frac{1}{2}},
\end{equation*}
which is from Corollary 2.3 in \cite{Bolley_and_Villani} (explicit form is Theorem \ref{thm: Bolley_and_Villani} in Section \ref{section: cite_theorem}). Also, since we have $\|\theta\|\le W$, $C_{\mu'}\le 2W$ holds.
We denote the density functions of $\mu_{\textbf{z}}^{(\beta,\varepsilon)},\mu_{\textbf{z}'}^{(\beta,\varepsilon)}$ as $p_\textbf{z},p_{\textbf{z}'}$, and the normalization constants as $\Lambda_\textbf{z},\Lambda_{\textbf{z}'}$ respectively. 
Let us calculate $D(\mu_{\textbf{z}}^{(\beta,\varepsilon)}||\mu_{\textbf{z}'}^{(\beta,\varepsilon)})$. 
We have
\begin{align}\label{KL-equation}
&\frac{p_\textbf{z}(\theta)}{p_{\textbf{z}'}(\theta)} =\frac{\Lambda_{\textbf{z}'}}{\Lambda_\textbf{z}}\cdot\frac{\beta M_{\ell}+1/\varepsilon+a_d\beta \|\nabla L_\textbf{z}(\theta)\|}{\beta M_{\ell}+1/\varepsilon+a_d\beta \|\nabla L_{\textbf{z}'}(\theta)\|}\exp\left(-\beta(L_\textbf{z}(\theta)-L_{\textbf{z}'}(\theta))\right),
\end{align}
so in order to obtain the upper bound of $D(\mu_{\textbf{z}}^{(\beta,\varepsilon)}||\mu_{\textbf{z}'}^{(\beta,\varepsilon)})$, we suppress each of the three terms of the right-hand side of \eqref{KL-equation}. 
First, we suppress the second term.
\begin{align*}
\|\nabla L_\textbf{z}(\theta)\|=&\left\|\nabla L_{\textbf{z}'}(\theta)+\frac{1}{n}(\nabla \ell(z_i;\theta)-\nabla \ell(\bar{z}_i;\theta)\right\| \\
\leq&\left\|\nabla L_{\textbf{z}'}(\theta)\right\|+\frac{1}{n}\|\nabla \ell(z_i;\theta)-\nabla \ell(\bar{z}_i;\theta)\| \\
\leq&\left\|\nabla L_{\textbf{z}'}(\theta)\right\|+\frac{2}{n}\left(c_1\|\theta\|+B\right),
\end{align*}
where the last inequality is from \eqref{upper of gradient}. Hence,
\begin{align}\label{KL-2}
\frac{\beta M_{\ell}+1/\varepsilon+a_d\beta \|\nabla L_\textbf{z}(\theta)\|}{\beta M_{\ell}+1/\varepsilon+a_d\beta \|\nabla L_{\textbf{z}'}(\theta)\|} \leq&\frac{\beta M_{\ell}+1/\varepsilon+a_d\beta\left(\left\|\nabla L_{\textbf{z}'}(\theta)\right\|+\frac{2}{n}\left(c_1\|\theta\|+B\right)\right)}{\beta M_{\ell}+1/\varepsilon+a_d\beta \|\nabla L_{\textbf{z}'}(\theta)\|} \notag\\
\leq&1+\frac{2a_d\beta(c_1W+B)}{n(\beta M_{\ell}+1/\varepsilon)} \notag\\
\leq&1+\frac{2a_d(c_1W+B)}{nM_{\ell}}
\end{align}
holds. 
Second, we suppress the third term. We have
\begin{align}\label{KL-3}
\exp\left(-\beta(L_\textbf{z}(\theta)-L_{\textbf{z}'}(\theta))\right) 
=&\exp\left(-\beta\left(\frac{1}{n}(\ell(z_i;\theta)-\ell(\bar{z}_i;\theta))\right)\right) \notag\\
\leq&\exp\left(\frac{\beta}{n}\left(\frac{c_1\|\theta\|^2}{2}+B\|\theta\|+A\right)\right) \notag\\
\leq&\exp\left(\frac{\beta}{n}\left(\frac{c_1W^2}{2}+BW+A\right)\right),
\end{align}
where we use \eqref{upper of loss}. 
Finally, we suppress the first term. Using \eqref{KL-2} and \eqref{KL-3}, we have
\begin{align}\label{KL-1}
\frac{\Lambda_{\textbf{z}'}}{\Lambda_\textbf{z}} =&\frac{\int_{\theta\in \Theta}\left(\beta M_{\ell}+1/\varepsilon+a_d\beta \|\nabla L_{\textbf{z}'}(\theta)\|\right)\exp\left(-\beta L_{\textbf{z}'}(\theta)\right)d\theta}{\int_{\theta\in \Theta}\left(\beta M_{\ell}+1/\varepsilon+a_d\beta \|\nabla L_{z}(\theta)\|\right)\exp\left(-\beta L_{z}(\theta)\right)d\theta} \notag\\
\leq&\left(1+\frac{2a_d(c_1W+B)}{nM_{\ell}}\right)\exp\left(\frac{\beta}{n}\left(\frac{c_1W^2}{2}+BW+A\right)\right).
\end{align}
Combining \eqref{KL-2}, \eqref{KL-3} and \eqref{KL-1}, we have
\begin{align*}
\log \frac{p_\textbf{z}(\theta)}{p_{\textbf{z}'}(\theta)} \leq&2\log\left(1+\frac{2a_d(c_1W+B)}{nM_{\ell}}\right)+\frac{2\beta}{n}\left(\frac{c_1W^2}{2}+BW+A\right) \\
\leq&\frac{1}{n}\left(\frac{4a_d(c_1W+B)}{M_{\ell}}+\beta(c_1W^2+2BW+2A)\right),
\end{align*}
so
\begin{equation*}
D(\mu_{\textbf{z}}^{(\beta,\varepsilon)}||\mu_{\textbf{z}'}^{(\beta,\varepsilon)})\leq \frac{1}{n}\left(\frac{4a_d(c_1W+B)}{M_{\ell}}+\beta(c_1W^2+2BW+2A)\right)
\end{equation*}
holds.
We set $C_d={4a_d(c_1W+B)} / {M_{\ell}} $ and $C=c_1W^2+2BW+2A$, then we have
\begin{equation}
\eqref{inequality1}\leq 2W(c_1W+B)\left(\left(\frac{C_d+\beta C}{n}\right)^{\frac{1}{2}}+\left(\frac{C_d+\beta C}{n}\right)^{\frac{1}{4}}\right). \label{ineq:2_bound}
\end{equation}

Finally, we evaluate \eqref{inequality3}. Let us denote
\begin{align*}
&\Lambda_{\textbf{z}}(\theta)=\frac{\Lambda}{\beta M_{\ell}+1/\varepsilon+a_d\beta\|\nabla L_{\textbf{z}}(\theta)\|} \\
&\Lambda=\int_{\theta\in \Theta} (\beta M_{\ell}+1/\varepsilon+a_d\beta\|\nabla L_{\textbf{z}}(\theta)\|)e^{-\beta L_{\textbf{z}}(\theta)}d\theta.
\end{align*}
Since the distribution of $\theta_{\mu}$ is
\begin{equation*}
\mu_{\textbf{z}}^{(\beta,\varepsilon)}(d\theta)\propto \left(\beta 
M_{\ell}+\frac{1}{\varepsilon}+a_d\beta \|\nabla L_{\textbf{z}}(\theta)\|\right)\exp(-\beta L_{\textbf{z}}(\theta)),
\end{equation*}
we have
\begin{align*}
\mathbb{E}_{\theta_{\mu}}[L_{\textbf{z}}(\theta_{\mu})] =&-\frac{1}{\beta}\left(\mathbb{E}_{\theta_{\mu}}\left[\log \frac{e^{-\beta L_{\textbf{z}}(\theta_{\mu})}}{\Lambda_{\textbf{z}}(\theta_{\mu})}\right]+\mathbb{E}_{\theta_{\mu}}[\log \Lambda_{\textbf{z}}(\theta_{\mu})]\right) \\
=&\frac{1}{\beta}\left(\mathbb{E}_{\theta_{\mu}}[-\log p_{\textbf{z}}(\theta_{\mu})]-\mathbb{E}_{\theta_{\mu}}[\log \Lambda_{\textbf{z}}(\theta_{\mu})]\right).
\end{align*}
Since we have $\mathbb{E}_{\theta_{\mu}}[\|\theta_{\mu}\|^2]\leq W^2$, we can calculate the upper bound of $\mathbb{E}_{\theta_{\mu}}[-\log p_{\textbf{z}}(\theta_{\mu})]$ by the differential entropy of Gaussian distribution in the same way as the discussion of Section 3.5 in \cite{Raginsky-2017}:
\begin{align*}
\mathbb{E}_{\theta_{\mu}}[-\log p_{\textbf{z}}(\theta_{\mu})] \leq \frac{d}{2}\log \left(\frac{2\pi e}{d}W^2\right).
\end{align*}
Using \eqref{upper of gradient}, we have
\begin{align*}
\log \Lambda_{\textbf{z}}(\theta) \geq \log \frac{\Lambda}{\beta M_{\ell}+1/\varepsilon+a_d\beta(c_1W+B))}.
\end{align*}
In addition,
\begin{align*}
\log \Lambda =&\log \int_{\theta\in \Theta} (\beta M_{\ell}+1/\varepsilon+a_d\beta\|\nabla L_{\textbf{z}}(\theta)\|)e^{-\beta L_{\textbf{z}}(\theta)}d\theta \\
\geq&\log \int_{\theta\in \Theta} (\beta M_{\ell}+1/\varepsilon)e^{-\beta L_{\textbf{z}}(\theta)}d\theta \\
=&\log(\beta M_{\ell}+1/\varepsilon)+\log\int_{\theta\in \Theta} e^{-\beta L_{\textbf{z}}(\theta)}d\theta \\
\geq&\log(\beta M_{\ell}+1/\varepsilon)-\beta L_{\textbf{z}}^*+\frac{d}{2}\log \frac{2\pi}{c_1\beta}
\end{align*}
holds, where the last inequality is from the equation (3.21) in \cite{Raginsky-2017}. Here, we denote $L_{\textbf{z}}^*=\inf_{\theta\in\Theta}L_{\textbf{z}}(\theta)$. Hence, we have
\begin{align}
\eqref{inequality3}\leq&\frac{1}{\beta}\left(\frac{d}{2}\log \left(\frac{2\pi e}{d}W^2\right)+\log\frac{\beta M_{\ell}+1/\varepsilon+a_d\beta(c_1W+B)}{\beta M_{\ell}+1/\varepsilon}+\beta L_{\textbf{z}}^*-\frac{d}{2}\log \frac{2\pi}{c_1\beta}\right)-L_{\textbf{z}}^* \\
\leq&\frac{1}{\beta}\left(\frac{d}{2}\log \frac{eW^2c_1\beta}{d}+\log \left(1+\frac{a_d(c_1W+B)}{M_{\ell}}\right)\right). \label{ineq:3_bound}
\end{align}

We combine the result \eqref{ineq:1_bound}, \eqref{ineq:2_bound}, and \eqref{ineq:3_bound}, then obtain the statement.
\end{proof}

\section{Proof of Proposition \ref{prop: log-Sobolev}}

\begin{proof}
Let $\theta_{\mu}$ and $\theta_{\nu}$ be the random variable which obey the distributions $\mu_{\textbf{z}}^{(\beta,\varepsilon)}$ and $\nu_{\textbf{z}}^{(\beta)}$ respectively.

In the same way as Theorem \ref{thm: generalization_error}, we have
\begin{align}
\mathbb{E}_z[\mathbb{E}_{\theta_K}[L(\theta_K)]]-\inf_{\theta\in \Theta} L(\theta) \label{inequality0'}
\leq&\mathbb{E}_z[\mathbb{E}_{\theta_K}[L(\theta_K)]-\mathbb{E}_{\theta_{\mu}}[L(\theta_{\mu})]] \\
\label{inequality1'}
+&\mathbb{E}_z[\mathbb{E}_{\theta_{\mu}}[L(\theta_{\mu})]-\mathbb{E}_{\theta_{\nu}}[L(\theta_{\nu})]] \\
\label{inequality2'}
+&\mathbb{E}_z[\mathbb{E}_{\theta_{\nu}}[L(\theta_{\nu})]-\mathbb{E}_{\theta_{\nu}}[L_{\textbf{z}}(\theta_{\nu})]] \\
\label{inequality3'}
+&\mathbb{E}_z[\mathbb{E}_{\theta_{\nu}}[L_{\textbf{z}}(\theta_{\nu})]-\inf_{\theta\in \Theta}L_{\textbf{z}}(\theta)].
\end{align}
\eqref{inequality0'} can be evaluated in the same way as Theorem \ref{thm: generalization_error}.

First, we evaluate \eqref{inequality1'}.
We have
\begin{align}
    \mathbb{E}_{\theta_{\mu}}[L(\theta_{\mu})]-\mathbb{E}_{\theta_{\nu}}[L(\theta_{\nu})]
    \leq W\mathcal{W}_2(\mu_{\textbf{z}}^{(\beta,\varepsilon)},\nu_{\textbf{z}}^{(\beta)})
\end{align}
from the same discussion in the proof of Theorem \ref{thm: generalization_error}.
Since both $\theta_{\mu}$ and $\theta_{\nu}$ satisfy the log-Sobolev inequality, we can use Otto-Villani theorem \citep{Bakry-2014} (explicit form is Theorem \ref{thm: Bakry-2014} in Section \ref{section: cite_theorem}), and 
\begin{align}
    \mathcal{W}_2(\mu_{\textbf{z}}^{(\beta,\varepsilon)},\nu_{\textbf{z}}^{(\beta)})
    \leq\sqrt{2c_\mathrm{LS}^{(\beta)}D(\mu_{\textbf{z}}^{(\beta,\varepsilon)}||\nu_{\textbf{z}}^{(\beta)})}
\end{align}
holds, where $D$ denotes the KL-divergence and $c_\mathrm{LS}^{(\beta)}$ is the log-Sobolev constant of $\nu_{\textbf{z}}^{(\beta)}$. We have
\begin{align}
    D(\mu_{\textbf{z}}^{(\beta,\varepsilon)}||\nu_{\textbf{z}}^{(\beta)}) =&\mathbb{E}_{\theta\sim\mu}\left[\log \frac{(\beta M_{\ell}+1/\varepsilon+a_d\beta\|\nabla L_{\textbf{z}}(\theta)\|)\exp\left(-\beta L_{\textbf{z}}(\theta)\right)/\Lambda_{\mu}}{\exp\left(-\beta L_{\textbf{z}}(\theta)\right)/\Lambda_{\nu}}\right] \\
    \leq&\mathbb{E}_{\theta\sim\mu}\left[\log (\beta M_{\ell}+1/\varepsilon+a_d\beta M_{\ell})\frac{\Lambda_{\nu}}{\Lambda_{\mu}}\right],
\end{align}
where $\Lambda_{\mu}$ and $\Lambda_{\nu}$ are normalizing constants of the density functions of $\mu_{\textbf{z}}^{(\beta,\varepsilon)}$ and $\nu_{\textbf{z}}^{(\beta)}$ respectively.
We have
\begin{align}
    \frac{\Lambda_{\nu}}{\Lambda_{\mu}} =&\frac{\int_{\Theta}\exp\left(-\beta L_{\textbf{z}}(\theta)\right)d\theta}{\int_{\Theta}(\beta M_{\ell}+1/\varepsilon+a_d\beta\|\nabla L_{\textbf{z}}(\theta)\|)\exp\left(-\beta L_{\textbf{z}}(\theta)\right)d\theta} 
    \leq \frac{1}{\beta M_{\ell}+1/\varepsilon},
\end{align}
hence we have
\begin{align}
    D(\mu_{\textbf{z}}^{(\beta,\varepsilon)}||\nu_{\textbf{z}}^{(\beta)})
    \leq\log\left(1+a_d\beta \varepsilon M_{\ell}\right).
\end{align}
As a result, we obtain
\begin{align}
    \mathbb{E}_{\theta_{\mu}}[L(\theta_{\mu})]-\mathbb{E}_{\theta_{\nu}}[L(\theta_{\nu})]\leq W\sqrt{2c_\mathrm{LS}^{(\beta)}\log\left(1+a_d\beta \varepsilon M_{\ell}\right)}. \label{ineq:4_bound}
\end{align}

Second, we evaluate \eqref{inequality2'}. Let $\nu_{\textbf{z}'}^{(\beta)}$ be the Gibbs distribution when one of the data $z_i$ is replaced by $z'_i$. In the same way as Section 3.6 in \cite{Raginsky-2017}, we have
\begin{align}
    \mathcal{W}_2(\nu_{\textbf{z}}^{(\beta)},\nu_{\textbf{z}'}^{(\beta)})\leq\frac{2c_\mathrm{LS}^{(\beta)}\beta M_{\ell}}{n}.
\end{align}
Hence, we have
\begin{align}
    \mathbb{E}_{\theta_{\nu}}[L(\theta_{\nu})]-\mathbb{E}_{\theta_{\nu}}[L_{\textbf{z}}(\theta_{\nu})]\leq (c_1W+B)\frac{2c_\mathrm{LS}^{(\beta)}\beta M_{\ell}}{n}.  \label{ineq:5_bound}
\end{align}

Finally, we evaluate \eqref{inequality3'}. This term can be evaluated in the same way as Proposition 3.4 in \cite{Raginsky-2017} and we have
\begin{align}
    \mathbb{E}_{\theta_{\nu}}[L_{\textbf{z}}(\theta_{\nu})]-\inf_{\theta\in \Theta}L_{\textbf{z}}(\theta)\leq&\frac{1}{\beta}\left(\frac{d}{2}\log\left(\frac{2\pi eW^2}{d}\right)-\frac{d}{2}\log\frac{2\pi}{c_1 \beta}\right) \\
    =&\frac{d}{2\beta}\log\left(\frac{eW^2c_1\beta}{d}\right).  \label{ineq:6_bound}
\end{align}

We combine the result \eqref{ineq:4_bound}, \eqref{ineq:5_bound}, and \eqref{ineq:6_bound}, then obtain the statement.
\end{proof}

\section{Proof of Theorem \ref{thm: global_convergence}}
\begin{proof}
Let $\theta_K,\theta_{\mu}$ be the random variables whose distribution is $\mu_{\textbf{z},K}^{(\beta,\varepsilon)}$ and $\mu_{\textbf{z}}^{(\beta,\varepsilon)}$ respectively. Let $L_{\textbf{z}}^*=\min_{\theta\in\Theta}L_{\textbf{z}}(\theta)$. We have
\begin{align}
    \mathbb{E}_{\theta_K}[L_{\textbf{z}}(\theta_K)]-L_{\textbf{z}}^* 
    =&(\mathbb{E}_{\theta_K}[L_{\textbf{z}}(\theta_K)]-\mathbb{E}_{\theta_{\mu}}[L_{\textbf{z}}(\theta_{\mu})])+(\mathbb{E}_{\theta_{\mu}}[L_{\textbf{z}}(\theta_{\mu})]-L_{\textbf{z}}^*).
\end{align}
As the first term of the right-hand side, we can use the Wasserstein distance in the same way as the proof of Theorem \ref{thm: generalization_error} as in \eqref{ineq:1_bound}. 
Hence, we have
\begin{align}
    &\mathbb{E}_{\theta_K}[L_{\textbf{z}}(\theta_K)]-\mathbb{E}_{\theta_{\mu}}[L_{\textbf{z}}(\theta_{\mu})]\leq(c_1W+B)\sqrt{Wd_K(\beta,\varepsilon,d)}.
\end{align}
Further, using \eqref{ineq:3_bound} in the Proof of Theorem \ref{thm: generalization_error},
\begin{align}
    \mathbb{E}_{\theta_{\mu}}[L_{\textbf{z}}(\theta_{\mu})] \leq&\frac{1}{\beta}\left(\frac{d}{2}\log \frac{eW^2c_1\beta}{d}+\log \left(1+\frac{a_d(c_1W+B)}{M_{\ell}}\right)\right)+L_{\textbf{z}}^*
\end{align}
holds, which completes the proof.
\end{proof}

\section{Explicit citation of the existing theorems}\label{section: cite_theorem}
\begin{theorem}[Theorem 4, \cite{Wasserstein-TV}]\label{thm: Wasserstein-TV}
    On the compact set $\Omega$, the Wasserstein metric $d_W$ and the total variation distance $d_{TV}$ satisfy the following relation:
    \begin{align}
        d_W\le\mathrm{diam}(\Omega)\cdot d_{TV},
    \end{align}
    where $\mathrm{diam}(\Omega) = \sup\{d(x,y)|x,y\in\Omega\}$.
\end{theorem}

\begin{theorem}[Corollary 2.3, \cite{Bolley_and_Villani}]\label{thm: Bolley_and_Villani}
    Let $X$ be a measurable space equipped with a measurable distance $d$, let $p\ge 1$ and let $\nu$ be a probability measure on $X$. Assume that there exist $x_0 \in X$ and $\alpha >0$ such that $\int e^{\alpha d(x_0,x)^p} d\nu(x)$ is finite. Then, $\forall \mu\in P(X)$, 
    \begin{align}
        W_p(\mu,\nu)\le C\left[H(\mu|\nu)^{\frac{1}{p}}+\left(\frac{H(\mu|\nu)}{2}\right)^{\frac{1}{2p}}\right],
    \end{align}
    where 
    \begin{align}
        C=2\inf_{x_0\in X,\alpha>0} \left(\frac{1}{\alpha}\left(\frac{3}{2}+\log \int e^{\alpha d(x_0,x)^p} d\nu(x)\right)\right)^{\frac{1}{p}}<\infty.
    \end{align}
\end{theorem}

\begin{theorem}[Theorem 9.6.1, \cite{Bakry-2014}]\label{thm: Bakry-2014}
    Let $\mu$ be a probability measure on $M$. If $\mu$ satisfies a logarithmic Sobolev inequality $LS(C)$ for some constant $C > 0$, then it satisfies following for every probability measure $\nu$ on $M$:
    \begin{align}
        \mathcal{W}_2(\mu,\nu)^2\le 2C\cdot D(\nu||\mu),
    \end{align}
    where $\mathcal{W}_2$ denotes the Wasserstein-2 distance and $D$ denotes the Kullback-Leibler divergence.
\end{theorem}

\end{document}